\documentclass[twoside]{article}

\usepackage{algorithm}
\usepackage{algorithmic}
\usepackage{fancyhdr}
\usepackage[accepted]{aistats2026}

\usepackage{url}
\usepackage{xcolor}
\definecolor{newcolor}{rgb}{.8,.349,.1}
\usepackage{tabularx,comment}
\usepackage{graphicx}
\usepackage{subcaption}
\usepackage{amsmath,amssymb,bm,bbm,mathtools}
\usepackage{verbatim}
\usepackage{enumitem}
\usepackage{epstopdf}
\usepackage[authoryear,round]{natbib}
\usepackage[colorlinks=true,citecolor=blue,linkcolor=blue,urlcolor=blue]{hyperref}

\usepackage{amsthm}

\newtheoremstyle{myplain}
  {3pt}
  {3pt}
  {\normalfont}
  {}
  {\bfseries}
  {.}
  { }
  {}

\theoremstyle{myplain}
\newtheorem{theorem}{Theorem}[section]
\newtheorem{lemma}[theorem]{Lemma}
\newtheorem{corollary}[theorem]{Corollary}

\newtheorem{definition}[theorem]{Definition}

\renewenvironment{proof}[1][Proof]{%
  \par\noindent\textbf{#1. }\ignorespaces
}{%
  \unskip\nobreak\hfill$\square$\par
}

\newcommand{\R}{\mathbb{R}}
\newcommand{\E}{\mathbb{E}}

\DeclareMathOperator*{\argmax}{arg\,max}
\DeclareMathOperator*{\argmin}{arg\,min}
\DeclareMathOperator*{\adv}{adv_A}

\let\epsilon=\varepsilon

\PassOptionsToPackage{numbers, compress}{natbib}

\usepackage{float}
\usepackage{placeins}   
\usepackage[utf8]{inputenc} 
\usepackage[T1]{fontenc}    
\usepackage{hyperref}       
\usepackage{url}            
\usepackage{booktabs}       
\usepackage{amsfonts}       
\usepackage{nicefrac}       
\usepackage{microtype}      
\usepackage{xcolor}         

\begin{document}

\twocolumn[

\aistatstitle{Learnability with  Partial Labels and Adaptive Nearest Neighbors }

\aistatsauthor{
Nicolas A. Errandonea$^{1}$ \ \ \ \ \ \
Santiago Mazuelas$^{1,2}$\ \ \ \ \ \
Jose A. Lozano$^{1,3}$ \ \ \ \ \ \
Sanjoy Dasgupta$^{4}$
}

\aistatsaddress{
$^1$ Basque Center for Applied Mathematics, Bilbao, Spain \\
$^2$ Ikerbasque, Basque Foundation for Science \\
$^3$  University of the Basque Country, Spain \\
$^4$ University of California San Diego, USA\\
\texttt{\{nerrandonea, smazuelas, jlozano\}@bcamath.org}, \ \texttt{sadasgupta@ucsd.edu}
}

]

\begin{abstract}
Prior work on partial labels learning (PLL) has shown that  learning is possible even  when each instance is associated with a bag of labels, rather than a single  accurate but costly label. However, the necessary conditions for learning with   partial labels  remain unclear, and existing PLL methods are effective only in specific scenarios. In this work, we mathematically characterize the settings in which PLL is feasible. In addition, we present PL A-$k$NN, an adaptive nearest-neighbors algorithm for PLL that is effective in general scenarios and enjoys strong performance guarantees. Experimental results corroborate that PL A-$k$NN  can outperform state-of-the-art methods in general PLL scenarios.

\end{abstract}

\section{Introduction}

Partial labels learning (PLL) is a weakly supervised  framework in which each training instance is associated with a bag of labels, rather than the ground-truth label \citep{nguyen2008classification,cour2011learning,tian2023partial}. The goal in PLL is to train a classifier that can accurately predict the true label for unseen instances, overcoming the ambiguity introduced by the bags. A typical example of partial labels arises in medical imaging annotation, where non-expert annotators may provide multiple labels when uncertain about the correct label of an image. Interest in PLL has continued to grow in recent years as it offers an attractive surrogate  to supervised learning, since accurately labeled datasets  often require costly expert annotation. PLL remains a challenging problem due to the diverse ways in which bags can be generated, often resulting in highly ambiguous training data.

Existing theoretical results for  PLL have established distributional conditions that are sufficient to recover the Bayes rule while learning  from bags of labels \citep{cour2011learning,liu2014learnability,cabannes2020structured,lv2020progressive,wen2021leveraged,lv2023robustness}. Among these,  \cite{cour2011learning} introduced the \hbox{label-aligned} condition as a realistic  and intuitive assumption for the bags of labels. Specifically, for label-aligned bags,  the most probable label of each instance is also the most probable label to appear in its bag of labels obtained at training. The label-aligned  condition encompasses most of the other proposed  assumptions as special cases.  Moreover, \cite{cour2011learning} suggested that  the label-aligned condition is also intuitively necessary for learning in partial labels settings, though no formal theoretical results were provided. Despite substantial progress  identifying sufficient conditions for recovering the Bayes rule, the fundamental properties that enable learning in a partial labels setting remain unclear.

Multiple  successful algorithmic approaches for PLL have been proposed in recent years. Existing  methods   obtain information from the training data by   either  trying to disambiguate the ground-truth label from each bag \citep{zeng2013learning,liu2012conditional,xu2023progressive,jia2024partial}, or treating all the labels in the bags as  ground-truth labels \citep{hullermeier2006learning,cour2011learning,zhang2015solving,zhou2018geometric,cabannes2020structured}. A popular approach among the latter methods is the
nearest neighbors \hbox{($k$-NN)} algorithm \citep{hullermeier2006learning}, which   predicts an instance’s label as the most frequent 
label among the $k$  neighboring bags.

  Existing PLL methods can obtain accurate classification and provide strong performance guarantees only in specific scenarios.  The approaches in \cite{hullermeier2006learning,cour2011learning,zeng2013learning,cabannes2020structured,feng2020provably,wen2021leveraged,xu2023progressive,jia2024partial} are based on the assumption that the bags of labels always contain the ground-truth label. This assumption  often breaks down in real-world scenarios, where annotators may  omit the correct label, producing  noisy partial labels \citep{cid2012proper}. The methods in \cite{feng2020provably,wen2021leveraged,lv2023robustness} rely on the assumption  that the bags' characteristics are the same for all the instances, thereby not considering  cases where the bag generation process varies across the  instance space. Moreover, the theoretical guarantees for  the existing    methods are  established  under  additional restrictive assumptions, such as assuming deterministic labels, and  only  cover  specific cases of label-aligned bags \citep{cour2011learning,cabannes2020structured,feng2020provably,wen2021leveraged,lv2023robustness}.

 In this work, we mathematically characterize the conditions that can enable  learning in a partial labels setting. In addition, we  introduce   PL A-$k$NN,  an  adaptive nearest neighbors algorithm  for partial labels that is   effective in general  scenarios, and enjoys strong theoretical guarantees. The main contributions of the paper are listed as follows:

 \begin{itemize}
\item We characterize the PLL settings in which it is possible to overcome the ambiguity introduced by the bags and recover the Bayes rule of the underlying distribution.
\item We propose PL A-$k$NN, an adaptive nearest-neighbor approach for PLL that tailors the number of neighbors for each instance according to the  bags of the surrounding instances. 

\item We show that PL A-$k$NN is Bayes consistent for any case  of label-aligned bags, and 
that no other algorithm can be  consistent under  more general PLL  scenarios  than  PL A-$k$NN. 

\item The experimental results demonstrate that  \hbox{PL A-$k$NN}  can outperform state-of-the-art methods in general scenarios. In addition, PL A-$k$NN  can achieve accuracies on pair with the ideal nearest neighbors method that uses the optimal number of neighbors.

  \end{itemize}    

\section{Learnability with partial labels and the label-aligned condition} \label{sec:identifibiality}

\subsection{Preliminaries}

The  instance space  is  considered to be a finite dimensional normed space, denoted as $\mathcal{X}$. Let \hbox{$\mathcal{Y}=\{1,2,\dots,c\}$} represent the label space and $\mathcal{S}=2^{\mathcal{Y}}$ represent the space of the  bags of labels. Data are assumed to be drawn i.i.d.  from  an unknown  distribution $P$ over $\mathcal{X}  \times \mathcal{Y} \times  \mathcal{S}$.  However, only pairs  of instances and  bags    $(x,s) \in \mathcal{X} \times  \mathcal{S} $ are available for learning. The  goal in PLL is to learn a classifier 
$h:\mathcal{X} \to \mathcal{Y}$ from a set of  $n$ training examples $\{(x_l,s_l)\}_{l=1}^n$.

The risk of a classification rule $h$ is the probability with which the rule misclassifies a tuple $(x,y)$ drawn i.i.d. from distribution $P$, that is:
\begin{align*}
    &\mathcal{R}(h)= \mathbb{E}_{(x,y) \sim P} \left[ \mathbb{I}\{h(x)\neq y\} \right]. 
\end{align*}
The Bayes risk $\mathcal{R^*}$ is the smallest possible risk and is achieved by the Bayes rule $h^*$, that predicts the most probable label for each instance. An algorithm is Bayes consistent if its  associated classifier asymptotically achieves the Bayes risk as the number of training instances grow.

For each instance $x$, the distribution  over bags  is linked to the label distribution $P(y|x)$ as follows
\begin{align*}
P(s|x)=\sum_{y\in \mathcal{Y}} P(s|y,x)  P(y|x).
\end{align*}
We refer to the conditional probabilities $P(s |y, x)$ as the bag generation process, since they describe how bags are generated given instance-label pairs.  A partial labels setting is determined by a bag generation process. 

For each instance $x\in\mathcal{X}$, the bag generation process is characterized by the  $|\mathcal{Y}|$ vectors \hbox{$\mathbf{p}_{1,x},\mathbf{p}_{2,x},\ldots,\mathbf{p}_{|\mathcal{Y}|,x}\in\mathbb{R}^{|\mathcal{S}|} $} given by
\begin{align} \label{vect}
\mathbf{p}_{i,x} &=
\begin{bmatrix}
P(s_1 | i, x) \\
P(s_2 | i, x) \\
\vdots \\
P(s_{|\mathcal{S}|} | i, x)
\end{bmatrix}
\in \mathbb{R}^{|\mathcal{S}|} \mbox{ for } i\in \mathcal{Y}.
\end{align}

 Each component of the vector corresponds to the probability of generating a bag in  $  \mathcal{S}=\{ s_1,s_2,...s_{|\mathcal{S}|}\}$ when the true label is $i$.

\subsection{ Learnability with partial labels}
Learning in a partial labels setting is not always feasible,  as even   an infinite amount of samples may not be enough  to  recover  the Bayes rule of the underlying distribution. Specifically,  two  probability distributions can share the same distribution over bags but yield  different Bayes rules.  Previous work has proposed several sufficient conditions for recovering the Bayes rule \citep{cour2011learning,liu2014learnability,cabannes2020structured,lv2023robustness}, but the  conditions that enable  learning with partial labels  remain unclear.

 As discussed above, the Bayes rule cannot be recovered in  partial labels  settings in which distributions with different Bayes rules result in the same distribution over bags.  Therefore, we  define a bag generation process  as reconstructible  if   any  two  probability  distributions  that abide by  it and   share     the same distribution over bags  also  yield the same Bayes rule. If a bag generation process is reconstructible, it is possible to recover the Bayes rule from the distribution over bags.

\begin{definition}
A bag generation process  $P(s|y,x)$ is   \textit{reconstructible} if  for all $x \in \mathcal{X}$ we have that  $\sum_{y\in \mathcal{Y}} P(s|y,x)Q_1(y)=\sum_{y\in \mathcal{Y}} P(s|y,x)Q_2(y)$ for    \hbox{$Q_1,Q_2 \in \Delta( \mathcal{Y})$} implies 
\[
\arg\max_{y \in \mathcal{Y}} Q_1(y)
=
\arg\max_{y \in \mathcal{Y}} Q_2(y).
\]
\end{definition}

The following theorem  characterizes  the necessary and sufficient condition that a bag generation process has to satisfy to be reconstructible. 

\begin{theorem} \label{reconstructibility}
A bag generation process $P(s | y, x)$ is reconstructible if and only if, the  vectors
$\mathbf{p}_{1,x}, \mathbf{p}_{2,x}, \dots, \mathbf{p}_{|\mathcal{Y}|,x}\in\mathbb{R}^{\mathcal{|S|}}$ in  \eqref{vect}  are linearly independent for any $x$.
\end{theorem}

 \begin{proof}

 We  define  $M(x)=[\mathbf{p}_{1,x}, \mathbf{p}_{2,x}, \dots, \mathbf{p}_{|\mathcal{Y}|,x}]\in\mathbb{R}^{\mathcal{|S| \times |\mathcal{Y}|}}$, the matrix that characterizes the bag generation process for instance $x$. The entries of the matrix  satisfy $M_{j,i}(x)=P(s_j|i,x)$ for any $x \in \mathcal{X}, i \in \mathcal{Y}, s_j\in \mathcal{S}  $.

\textbf{Sufficiency:} 
If for each $x \in \mathcal{X}$ the vectors 
\begin{align}\label{vect_p}
\mathbf{p}_{1,x}, \mathbf{p}_{2,x}, \dots, \mathbf{p}_{|\mathcal{Y}|,x} \in \mathbb{R}^{|\mathcal{S}|}
\end{align}
are linearly independent, and we have a pair of distributions $Q_1, Q_2 \in \Delta(\mathcal{Y})$ such that
\[
\sum_{y \in \mathcal{Y}} P(s| y,x)\, Q_1(y) \;=\; 
\sum_{y \in \mathcal{Y}} P(s | y,x)\, Q_2(y) 
\quad \forall s \in \mathcal{S},
\]
then  $M(x)(Q_1 - Q_2) = 0$. Since  the vectors in \eqref{vect_p}   are linearly independent, $M(x)$ has full column rank, which implies that $Q_1 = Q_2$. In particular,
\[
\arg\max_{y \in \mathcal{Y}} Q_1(y) = \arg\max_{y \in \mathcal{Y}} Q_2(y),
\]
hence the bag generation process is reconstructible.

\textbf{Necessity:} 
This proof is obtained by contradiction. If  the bag generation process is reconstructible, and  for some $x_0 \in \mathcal{X}$,  the vectors
\[
\mathbf{p}_{1,x_0}, \mathbf{p}_{2,x_0}, \dots, \mathbf{p}_{c,x_0}
\]
are linearly dependent, then there would exist a nonzero vector $q \in \mathbb{R}^c$ such that $M(x_0) q = 0$. 

Since the columns of $M(x_0)$ add to one, we also have that $\sum_{y\in \mathcal{Y}}q=0$. Then, we could split $q$ into its positive-negative parts:
\begin{align*}
q_+(y) &= \max\{0, q(y)\}, \\
q_-(y) &= \max\{0, -q(y)\}, \quad y \in \mathcal{Y}.
\end{align*}
Clearly, $q = q_+ - q_-$, and since $\sum_{y\in \mathcal{Y}}q=0$,  we have that both $q_+$ and  $q_-$ are nonzero.  

Normalizing the   vectors $q_+$ and  $q_-$  would give two different  label distributions:
\[
Q_1 = \frac{q_+}{\sum_{y \in \mathcal{Y}} q_+(y)}, \qquad 
Q_2 = \frac{q_-}{\sum_{y \in \mathcal{Y}} q_-(y)} \in \Delta(\mathcal{Y}),
\]

in which $\sum_{y \in \mathcal{Y}} q_+(y)=\sum_{y \in \mathcal{Y}} q_-(y)$, as $\sum_{y\in \mathcal{Y}}q=0$. Then, we would have
\[
M(x_0) (Q_1 - Q_2) \;=\; \frac{1}{\sum_y q_+(y)} M(x_0) q \;=\; 0.
\]
Therefore, we would have $  M(x_0) Q_1 = M(x_0) Q_2$, or equivalently 
\[
\sum_{y\in \mathcal{Y}} P(s | y,x_0)\, Q_1(y)= \sum_{y\in \mathcal{Y}} P(s | y,x_0)\, Q_2(y)\quad \forall s \in \mathcal{S}.
\]
 However, since $q_+,q_-$ are the positive-negative parts  of $q$ we would have
\[
\argmax_{y\in \mathcal{Y}} Q_1(y)  \neq \argmax _{y\in \mathcal{Y}}Q_2(y),
\] which contradicts the reconstructibility of $P(s|y,x)$.\end{proof}

The theorem above presents the necessary and sufficient  condition for learning in  a partial labels setting, establishing that the bag generation process has to satisfy a linear independence condition. Theorem \ref{reconstructibility} provides the first complete characterization of learnability with partial labels. Previous work in \cite{feng2020provably} already proved that  the linear independence of $\mathbf{p}_{1}, \mathbf{p}_{2}, \dots, \mathbf{p}_{|\mathcal{Y}|}\in\mathbb{R}^{\mathcal{|S|}}$ is sufficient for learning when the bag generation process does not vary across instances. Theorem \ref{reconstructibility} generalizes the result from \cite{feng2020provably} and shows that the linear independence assumption is in fact a necessary condition. In contrast, the other existing sufficiency results \citep{liu2014learnability,cabannes2020structured} are oriented  to specific  assumptions that enable the recovery of the Bayes rule with a given procedure.

 Contrary to the suggestion of \cite{cour2011learning}, \hbox{Theorem  \ref{reconstructibility}} shows that the label-aligned  condition is  not necessary for learning in a  partial labels setting. Specifically, there exist bag generation processes that are reconstructible but do not correspond to label-aligned scenarios. For instance, the process defined by $P(s | y, x) = 1$ if $s = \{\pi_x(y)\}$ for some permutation $\pi_x$ of the labels, and $P(s | y, x) = 0$ otherwise, is reconstructible, since it satisfies  the conditions of \hbox{Theorem~\ref{reconstructibility}}. Such bag generation process would only be label-aligned if $\pi_x$ is the identity permutation for any instance $x$.  However, even though such a case is reconstructible, it cannot be realistically addressed by an algorithm,  since it would require knowing the permutation  $\pi_x$ for every instance $x$. 

As discussed above, full access to the bag generation process $P(s| y, x)$ can enable effective learning in any reconstructible setting, but such a detailed knowledge is highly unrealistic in practice. Therefore, we seek more practical assumptions under which an algorithm can recover the Bayes rule. The following further analyses the label-aligned condition, as  a simple general assumption that does not require  detailed knowledge of the bag generation process.

\subsection{Label-aligned assumption}
The label-aligned scenarios of PLL satisfy that, for each instance, the label that appears more frequently in its bags is also the most probable label. Specifically, let 
\begin{align*}
    \mathcal{S}_y=\{ s\in \mathcal{S}: y\in s \}
\end{align*}
be the set of bags that include  label $y$. The label-aligned condition requires that 
\begin{align}\label{partial_condition}
    &\argmax_{y\in \mathcal{Y}}P(\mathcal{S}_y|x) =  \argmax_{y\in \mathcal{Y}}P(y|x) \hspace{3mm} \forall x \in \mathcal{X}, \\
   &\text{where} \hspace{5mm} P(\mathcal{S}_y|x)= \sum_{s\in \mathcal{S}_y}P(s|x).\nonumber
\end{align}

The following result indicates that bag generation processes  that result in label-aligned scenarios are always reconstructible

\begin{definition}
    A bag generation process $P(s|y,x)$ is \textit{label-aligned} if any distribution  where the bags are generated according to  $P(s|y,x)$ satisfies the  label-aligned \hbox{condition \eqref{partial_condition}}.
\end{definition}

\begin{corollary} \label{obv}If a bag generation process $P(s|y,x)$ is label-aligned,  then it is also reconstructible. \label{bag_aligned}
\end{corollary}
\begin{proof} See Appendix C.\end{proof}

The label-aligned condition  is   more general than most of the assumptions taken  in the literature. In \hbox{Appendix A} we show  that the assumptions taken for   six partial-label models from \cite{cour2011learning,liu2014learnability,cabannes2020structured,feng2020provably,wen2021leveraged,lv2023robustness} are specific cases of the label-aligned condition. Furthermore, Theorem \ref{consistency} in Section 3 establishes that the label-aligned condition  is a   minimal algorithmic assumption.

In the following section  we introduce PL A-$k$NN, an effective algorithm in  general label-aligned scenarios. PL A-$k$NN  predicts the most probable label of an instance by computing the  empirical frequencies of the  labels in the bags of the  neighbors.

\section{\texorpdfstring{The PL A-$k$NN algorithm}{The PL A-kNN algorithm}} \label{sec:algorithm}

This section first details the proposed PL A-$k$NN algorithm. Then, we provide a consistency result for PL A-$k$NN,  along with rates of convergence, for any label-aligned case. We conclude the section showing the robustness of the algorithm in cases that depart from the label-aligned condition.

The PL A-$k$NN  algorithm  gradually whittles down the potential labels of an instance    by  comparing the frequencies of the   labels among the bags of an increasing set of neighbors. The algorithm uses  fewer neighbors when the most frequent label has a clear lead over the  frequencies of the other labels in the bags, and   uses more neighbors when the label frequencies are close.

  PL A-$k$NN  initially  considers   the  set with all  labels $\hat{s}=\mathcal{Y}$ as the potential labels of  instance $x$. At step $k$, the algorithm computes  the frequencies of  the labels in $\hat{s}$ over the  bags of the $k$ nearest neighbors of $x$. Labels with frequencies that  deviate from the maximum frequency by more than   threshold 
\begin{align*}
   & \Delta(n,k,\delta)= c_1\sqrt{\frac{ \log(n) + \log(|\mathcal{Y}|/\delta)}{k}}  
\end{align*}
are removed from $\hat{s}$. The pseudo-code for  PL A-$k$NN  is provided in  Algorithm \ref{alg1}.

\begin{algorithm}[H]
\caption{\textbf{PL A-$k$NN algorithm}}\label{alg1}
 \textbf{Inputs:}  Instance  $x$   \\ 
 Training examples $\{ (x_l, s_{l})\}_{l=1}^n $ \\
 Maximum number of 
 iterations $T$\\
   Confidence parameter $\delta$  \\
\textbf{Output:}  Label  $h(x) $

\begin{algorithmic}[1]
\STATE Set $A=c_1 \sqrt{\log(n) +\log(|\mathcal{Y}|/\delta})$  
\STATE Initialize $\hat{s}=\mathcal{Y}$, $k=0$, and  $\tau_1,\tau_2 ... ,\tau_{|\mathcal{Y}|} =0$ 
\WHILE{$|\hat{s}|>1$ and $k< T $} 
\STATE $k=k+1$
\STATE Find the $k$ nearest neighbor of $x$ and take $l_k$ as its index in $l=1,2,...n$
\STATE $\Delta=\frac{A}{\sqrt{k}}$
\STATE $\tau_y=\tau_y + \mathbb{I}\{ y \in s_{l_k} \}$ \hspace{3mm} $\forall y\in \mathcal{Y}$

\STATE $m= \max_{y \in \hat{s}} \tau_y$ 
\FOR{$ y \in  \hat{s}$}

\IF{$ \frac{m-\tau_y}{k}\geq \Delta $} 
\STATE $\hat{s}=\hat{s} \setminus \{y\}$
\ENDIF
\ENDFOR

\ENDWHILE
\STATE $h(x) = \hat{s}$

\end{algorithmic}
\end{algorithm}

The algorithm uses a maximum number of iterations $T$, which is set to a value significantly smaller than the number of samples, since bags  from very distant neighbors are poor representatives of the bag corresponding to $x$. If after $T$ iterations $|\hat{s}| > 1$, we require a disambiguation criterion. In this paper we propose a simple disambiguation process in which we  select the label in $\hat{s}$ that came closest to eliminating all other labels during the iterative process. The detailed pseudo code for PL A-\(k\)NN with the disambiguation criterion is provided in Appendix~D.

The  PL A-$k$NN algorithm is inspired by the adaptive $k$-nearest neighbor method  for binary supervised classification (A-$k$NN)  from  \cite{balsubramani2019adaptive}. PLL requires to consider multiclass settings, and the proposed PL A-kNN differs fundamentally from the multiclass  supervised classification extension suggested in \cite{balsubramani2019adaptive}. Such an  extension expands the neighborhood size until the  frequency of a label is higher than threshold $\Delta$, and then  predicts that label. In contrast, PL A-$k$NN gradually enlarges the neighborhood while discarding  labels  that deviate from the  maximum frequency   by more than   threshold  $\Delta$, narrowing down  the potential labels of the instance. Moreover, PL A-$k$NN  is backed by   strong theoretical results, whereas the extension in \cite{balsubramani2019adaptive}   is not provided with any guarantee, since the theoretical results in that work are only valid for the binary method.  The  difference in performance guarantees lies in how the two methods exploit label frequencies:
PL A-$k$NN leverages the margins between label frequencies,
while A-$k$NN  only considers which is the most frequent label and disregards the rest of the information.
A more detailed comparison between these two approaches is provided in Appendix~B.

\subsection{\texorpdfstring{Consistency of PL A-$k$NN}{Consistency of PL A-kNN}}
The following theorem shows that the  \hbox{PL A-$k$NN} algorithm is  Bayes consistent. Specifically, the risk of   PL A-$k$NN   converges almost surely to the Bayes risk  as the number of samples $n$ grows. The sequence of the  confidence parameters $\{ \delta_n\}_{n=1}^\infty$  and maximum number of iterations  $\{ T_n\}_{n=1}^\infty$  has to satisfy the following asymptotic properties
\begin{equation}
\label{conf_seq}
\begin{split}
\sum_{n=1}^{\infty} \delta_n < \infty, \quad
\lim_{n \to \infty} \frac{\log (1/\delta_n)}{T_n} = 0, \\
\exists A \in (0,1] \text{ such that } T_n \ge A n \ \forall n.
\end{split}
\end{equation}

\begin{theorem} \label{consistency}
  (Consistency) 
  Let $h_n $ be     the  classifier from the PL A-$k$NN Algorithm \ref{alg1}  with parameters that satisfy \eqref{conf_seq}. If  the underlying probability distribution satisfies the  label-aligned  condition \eqref{partial_condition},   we have 
   $$\lim_{n\rightarrow \infty}\mathcal{R}(h_n) = \mathcal{R}^* $$
    almost surely. In addition, no other   algorithm can be Bayes consistent under  more general PLL scenarios   than PL A-$k$NN.
\end{theorem}
\begin{proof} See Appendix C.\end{proof}

 The previous theorem  shows  that PL A-$k$NN is  Bayes consistent  under more general assumptions than  the other existing methods in the literature. Specifically, Appendix A shows that  the assumptions taken for the consistency guarantees of  the algorithms from \cite{cour2011learning,cabannes2020structured,feng2020provably,wen2021leveraged} are   specific cases of the label-aligned condition. There is one model proposed in \cite{lv2023robustness} that  shows consistency  for  cases where the bags are not label-aligned. However, such work assumes that the bag generation process is identical across instances, overlooking scenarios in which the bag generation varies across the instance space, and establishes the consistency result only under deterministic labels

 The previous theorem also demonstrates  that  \hbox{PL A-$k$NN} is consistent under  minimal algorithmic assumptions. In other words, an algorithm cannot achieve  Bayes consistency under  strictly more general  scenarios than those where  PL A-kNN is consistent.  To the best of our knowledge,  \hbox{PL A-$k$NN} is the first  algorithm for PLL that is  proven consistent under minimal algorithmic assumptions.
 
\subsection{Rates of convergence } 

In what follows, we first establish instance-specific convergence rates, and subsequently derive the standard convergence rates. These rates   depend on the  advantage of  each instance $x$, which quantifies the margin of the most frequent labels ($ \argmax_{y\in Y} P(\mathcal{S}_y|x)$)  with the other  labels in the bags surrounding $x$. The  concept of advantage  we  propose in this paper is an extension to PLL of the   concept introduced in \cite{balsubramani2019adaptive}  for binary supervision.

The advantage of an instance $x$ depends on the frequencies of  labels in  the bags for instances in balls centered at $x$. We denote the  closed ball $B$ centered in $x \in \mathcal{X}$ with radius $r$  as:
\begin{align*}
    B(x,r)= \{ x' \in \mathcal{X} :  \lVert x-x'  \rVert \leq r \}.
\end{align*}
We then  denote as $r_{p}(x)$  the  smallest radius such that ball $B(x,r_{p}(x))$ has probability mass of at least $p$, that is:
\begin{align*}
r_{p}(x) = \inf\{r \geq0 : P(B(x,r)) \geq p  \}.
\end{align*}
A point $x \in \mathcal{X}$ is $(p, \gamma)$-salient  if the following holds:
\begin{align*}
& \text{for any  } i \in \argmax_{y\in \mathcal{Y}} P(\mathcal{S}_y|x) \text{\ and \ } j \notin \argmax_{y\in \mathcal{Y}}P(\mathcal{S}_y|x), \hspace{3mm}  \\
    &   P(\mathcal{S}_i|B(x,r)) >P(\mathcal{S}_j|B(x,r)) +\gamma   \hspace{5mm}
\forall r \in [0,r_p(x)]. 
\end{align*}

A point $x$ can satisfy this definition for a variety of  $(p, \gamma)$  tuples.  The  advantage of $x \in \mathcal{X}$ in the region $B(x,r_{A}(x))$  is taken to be the largest value of $p\gamma^2$ out of all the $(p, \gamma)$-salient tuples with $p\leq A$, that is
\begin{equation*} \label{eq:adv}
\adv(x) =
\begin{dcases}
    1 \hspace{10mm}  \text{if } \arg\max_{y \in Y} P(\mathcal{S}_y \mid x) = \mathcal{Y}, \\[6pt]
    \sup \{ p \gamma^2 : x \text{ is } (p,\gamma)\text{-salient}, p\leq A \} \hspace{2mm} \text{else}.
\end{dcases}
\end{equation*}

Provided that \(A\) is not too small, the advantage within  $B(x,r_{A}(x))$  is  generally the same as in the entire space (the ball of mass $1$), since  instances are commonly \hbox{\((p,\gamma)\)-salient} only within a local region of the  space. Moreover, it always  holds that at least $\text{adv}_1(x) \le \frac{\adv(x)}{A}$.

Large advantage values for $x$ indicate that it  is easier  to distinguish the labels that are the most  frequent  labels in the bags of $x$ with the bags of neighboring instances.  

The following theorem shows that  PL A-$k$NN  classifies  $x$ as the Bayes rule  $h^*(x)$  with probability at least $1-\delta^2$  when the number of samples  is of the order of  $1/\adv(x)$, with  $A$ given by \eqref{conf_seq}.

\begin{theorem} \label{th1}
(Query-specific convergence) There is an absolute  constant $C>0$ for which the following holds. Let $h$ be the classifier from  PL A-$k$NN  \hbox{Algorithm \ref{alg1}}  with   maximum number of iterations $T$ and   confidence parameter $\delta $ satisfying \eqref{conf_seq}.  If the underlying probability distribution satisfies the  label-aligned   \hbox{condition \eqref{partial_condition}}, for each $x\in \mathcal{X}$  such that 
\begin{align} \label{avd_cond}
n \geq \frac{C}{\adv(x)} \, 
      \max\Big\{ \log\frac{1}{\adv(x)}, \, \log\frac{|\mathcal{Y}|}{\delta} \Big\}.
\end{align}

we have     $h(x)=h^*(x)$  with probability at least $1- \delta^2$.
In addition, the set of instances $x$ for which $\adv(x)=0$ has zero measure.
\end{theorem}
\begin{proof} See Appendix C.\end{proof}

 The  theorem above directly relates the advantage of an instance   with the number of  samples required by  \hbox{PL A-$k$NN} to ensure its correct classification. The criterion \eqref{avd_cond}  of the theorem  assesses  whether there are   enough  "advantageous" training examples  in the neighborhood of $x$ for  \hbox{PL A-$k$NN}  to predict the most most probable label for $x$ after at most $T$ iterations.

 The query-specific rates from Theorem \ref{th1} are tailored to each instance, depending only on the local properties of the bags.  On the other hand, the rates of convergence of other methods \citep{cabannes2020structured,feng2020provably,lv2023robustness} are  with  respect to  the overall  risk of the classifier.

 The bounds of the  Theorem  \ref{th1} hold for each instance $x$ with high probability. We can strengthen this theorem so that the guarantee holds with high probability simultaneously for all elements in $\mathcal{X}$ by modifying the threshold for candidate elimination to
\begin{align} \label{delta_unif}
   & \Delta(n,k,\delta)= c_1\sqrt{\frac{ d_0\log(n) +\log(|\mathcal{Y}|/\delta)}{k}} 
\end{align}
where $d_0$ denotes the VC dimension of   the set of balls in $\mathcal{X}$.

\begin{theorem}\label{uniform}
(Uniform query-specific convergence) Let $h$ be the classifier from PL A-$k$NN \hbox{Algorithm \ref{alg1}}, with maximum number of iterations $T$  and $\delta$ satisfying \eqref{conf_seq}, and $\Delta$ taken as \hbox{in \eqref{delta_unif}}. If the underlying probability distribution satisfies the label-aligned condition \eqref{partial_condition}, then, with probability at least \hbox{$1-\delta^2$},  we have that $h(x)$ coincides with the Bayes rule $h^*(x)$ for all $x\in \mathcal{X}$ that satisfy \eqref{avd_cond}.
\end{theorem}

\begin{proof} See Appendix C.\end{proof}

Theorem  \ref{uniform} shows that the  PL A-$k$NN  algorithm  with high probability can correctly classify all the instances that satisfy \eqref{avd_cond}.

It is also possible to derive instance-specific rates  for  the elimination of each  label that is not the most probable. Specifically, one can follow arguments analogous to those developed in this section by defining the advantage of the most probable label relative to each other label individually.  Therefore, even for instances where the available sample size is insufficient to guarantee the  identification of the optimal label, a similar per label analysis can ensure that certain suboptimal labels are eliminated from the set of possible labels.

We  now provide  rates of convergence  to the Bayes risk  for the PL A-$k$NN  classifier  that depend  on the cumulative distribution of the advantage.

\begin{theorem} \label{rates}
  (Rates of convergence) Let $C$ be the constant from Theorem \ref{th1} and  $h$ be the classifier from  PL A-$k$NN Algorithm \ref{alg1} with  maximum number of iterations $T$   and  confidence parameter $\delta$ satisfying \eqref{conf_seq}.  If the underlying probability distribution  satisfies the label-aligned condition \eqref{partial_condition},  with probability at least $1-\delta$, $h$ satisfies  
\begin{align*} 
& R(h)-R^* \leq \delta + P(\adv(x) \leq a_n)\\ 
& \textit{where } \hspace{1mm} a_n=\frac{C}{n} \max \bigl\{2\log( n),\log(|\mathcal{Y}|/\delta) \bigl\} \end{align*}

\end{theorem}
\begin{proof} See Appendix C.\end{proof}

Theorem \ref{rates} bounds the  excess risk of the  PL A-$k$NN  classifier with the Bayes risk in terms of the cumulative distribution of the advantage at $a_n$. Namely, the excess risk of  PL A-$k$NN  is upper bounded  by the probability mass  of the   instances  with advantage smaller  than $a_n$. 

The cumulative distribution of the advantage   can be seen   as a suitable notion of the "global  ambiguity" in a specific scenario.  Steeper functions  for  this  cumulative distribution  mean that 
 fewer instances  have  small advantage,  thus providing faster rates of convergence. Therefore, \hbox{Theorem \ref{rates}}  provides rates of  convergence tailored to the PLL scenario, as the rates only  depend on  the  global ambiguity  of the underlying distribution. On the other hand, the rates of convergence provided for other methods \citep{cabannes2020structured,feng2020provably,lv2023robustness}  are given in terms of   the Rademacher complexity of the  family  of classifiers considered and the maximum ambiguity in the scenario, thereby not fully  adapting  the bounds to the underlying distribution. 
 
 In Appendix B we provide explicit finite-sample rates of convergence by assuming two  common smoothness conditions  for non-parametric estimators over $P$.

\subsection{\texorpdfstring{Robustness of PL A-$k$NN under relaxed label-aligned conditions}{Robustness of PL A-kNN under relaxed from label-aligned conditions}}

The  label-aligned  condition \eqref{partial_condition}   might  not hold 
in some  real-life scenarios. For example, the condition  might not be satisfied in  a   small subset of instances. In other cases, the  most frequent  labels in  the bags  might not be  the most probable labels but  have  conditional probabilities   close to the maximum. Therefore, we weaken the original condition \eqref{partial_condition} and  bound the behavior of the limiting risk of the PL A-$k$NN  classifier in these cases.

For any instance $x$ and threshold $\theta \in [0,1]$, we define  as $\mathcal{Y}^\theta(x)$ the set of labels whose conditional probability are within a margin $\theta$ of the most probable label. Specifically, 
\begin{align*}
    & \mathcal{Y}^\theta(x)= \bigl\{ i \in \mathcal{Y} :  \hspace{1mm} P(j|x)- P(i|x)\leq \theta \hspace{3mm} \forall j \in \mathcal{Y} \bigl\}.
\end{align*}
 We assume that  bags of labels  satisfy the following relaxed label-aligned condition for a subset $G\subseteq \mathcal{X}$ and $ \theta \in [0,1]$:
\begin{equation} \label{relax}
\begin{alignedat}{2}
    \arg\max_{y \in \mathcal{Y}} P(\mathcal{S}_y |x) 
        &= \arg\max_{y \in \mathcal{Y}} P(y | x) 
        && \hspace{4mm} \forall x \in \mathcal{X}\setminus G \\[6pt]
    \arg\max_{y \in \mathcal{Y}} P(\mathcal{S}_y | x) 
        &\subseteq \ \mathcal{Y}^\theta(x),
        && \hspace{4mm} \forall x \in G \vphantom{\arg\max_{y \in \mathcal{Y}} P(\mathcal{S}_y \mid x)}
\end{alignedat}
\end{equation}

\begin{figure*}[!t]
    \centering
    \begin{subfigure}[b]{0.48\linewidth}
        \centering
        \includegraphics[width=\linewidth]{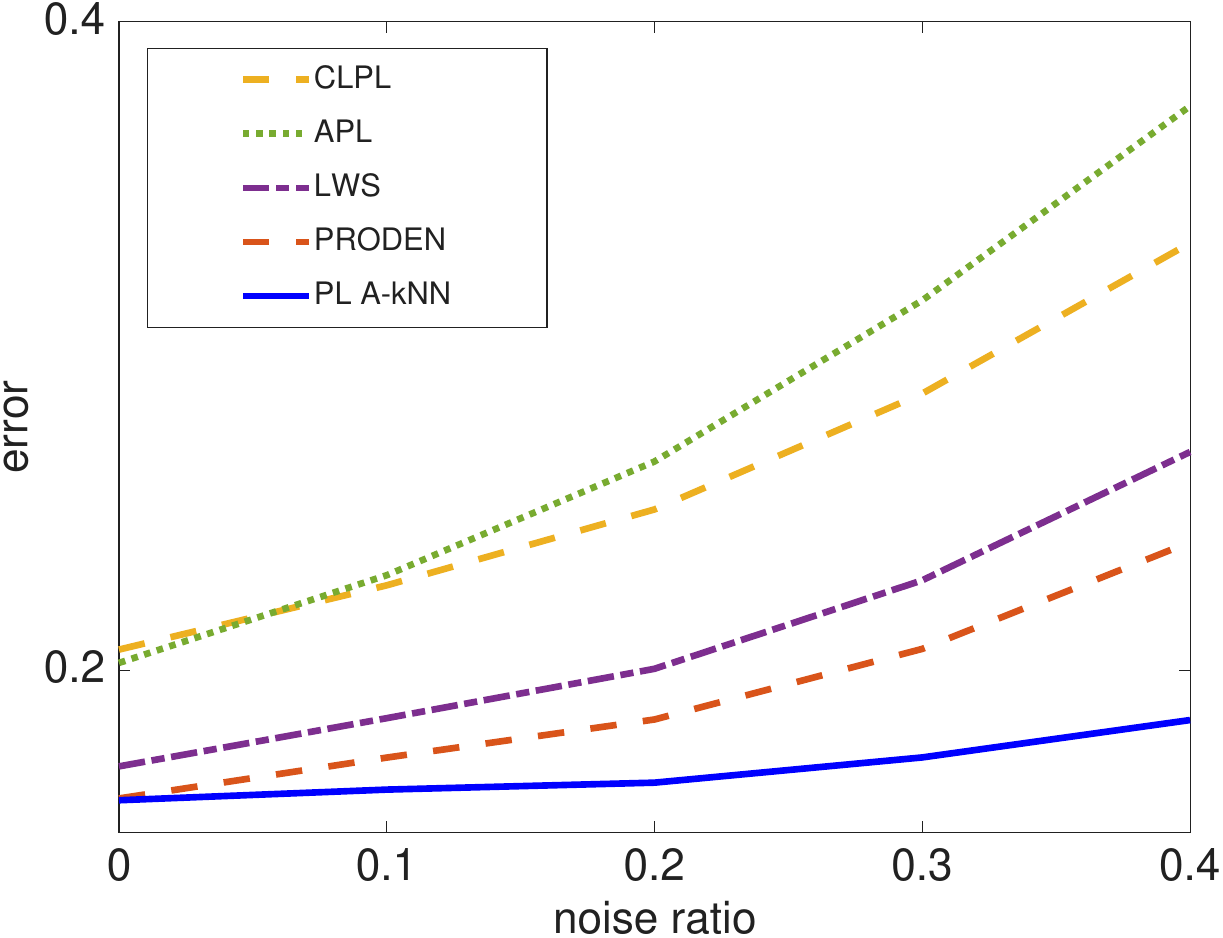}
        \caption{ Fashion-MNIST}
    \end{subfigure}
    \hspace{0.02\linewidth}
    \begin{subfigure}[b]{0.48\linewidth}
        \centering
        \includegraphics[width=\linewidth]{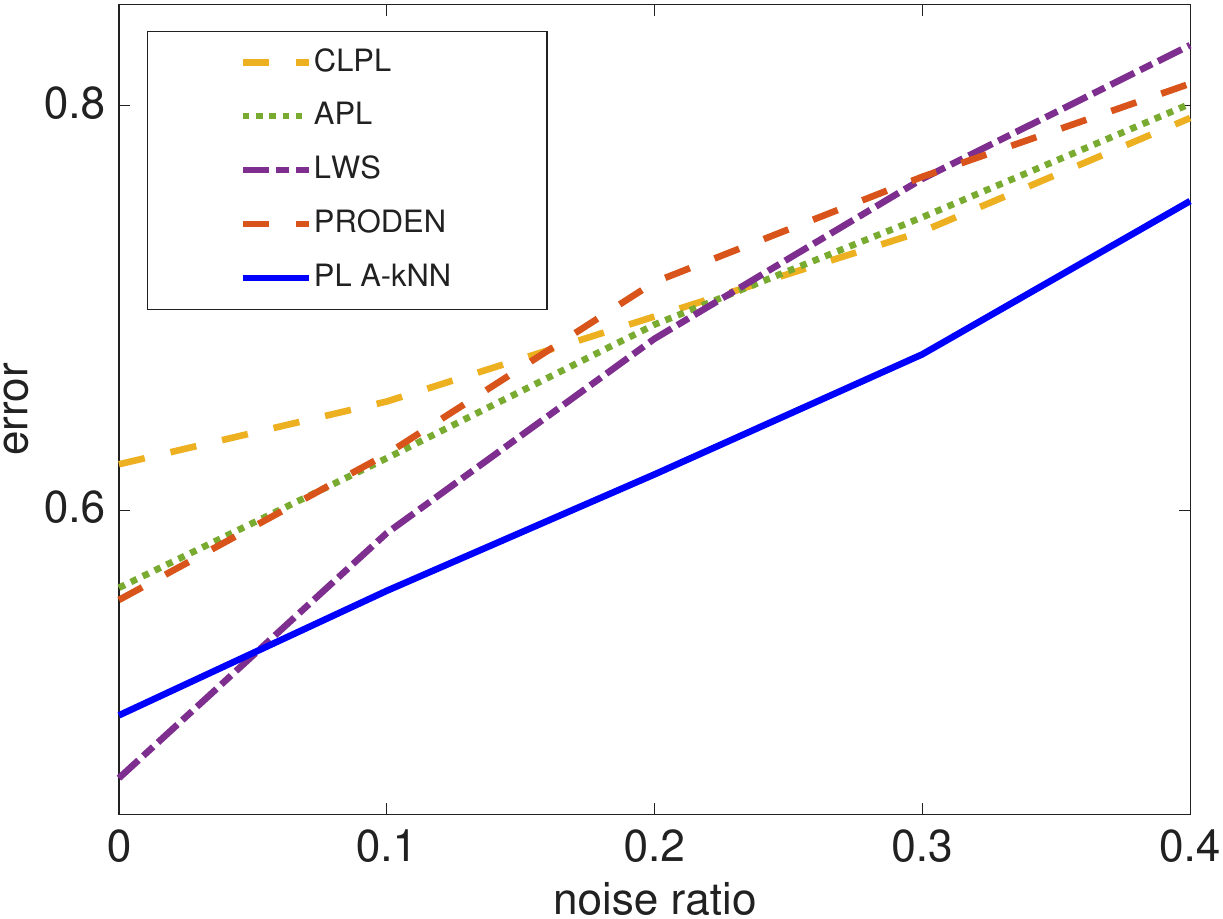}
        \caption{MSRCV2}
    \end{subfigure}
    \caption{Comparison of the error rates of PL A-$k$NN and state-of-the-art methods for  Fashion-MNIST and MSCRv2 under  an increasing noise rate. The results show that PL A-$k$NN consistently outperforms existing approaches across a wide range of noise levels. See Appendix D for the comparison in MNIST CIFAR10 and MirFlickr.}
    \label{results_sota}
\end{figure*}

\begin{figure*}[!t]
    \centering
    \begin{subfigure}[b]{0.48\textwidth}
        \centering
        \includegraphics[width=\linewidth]{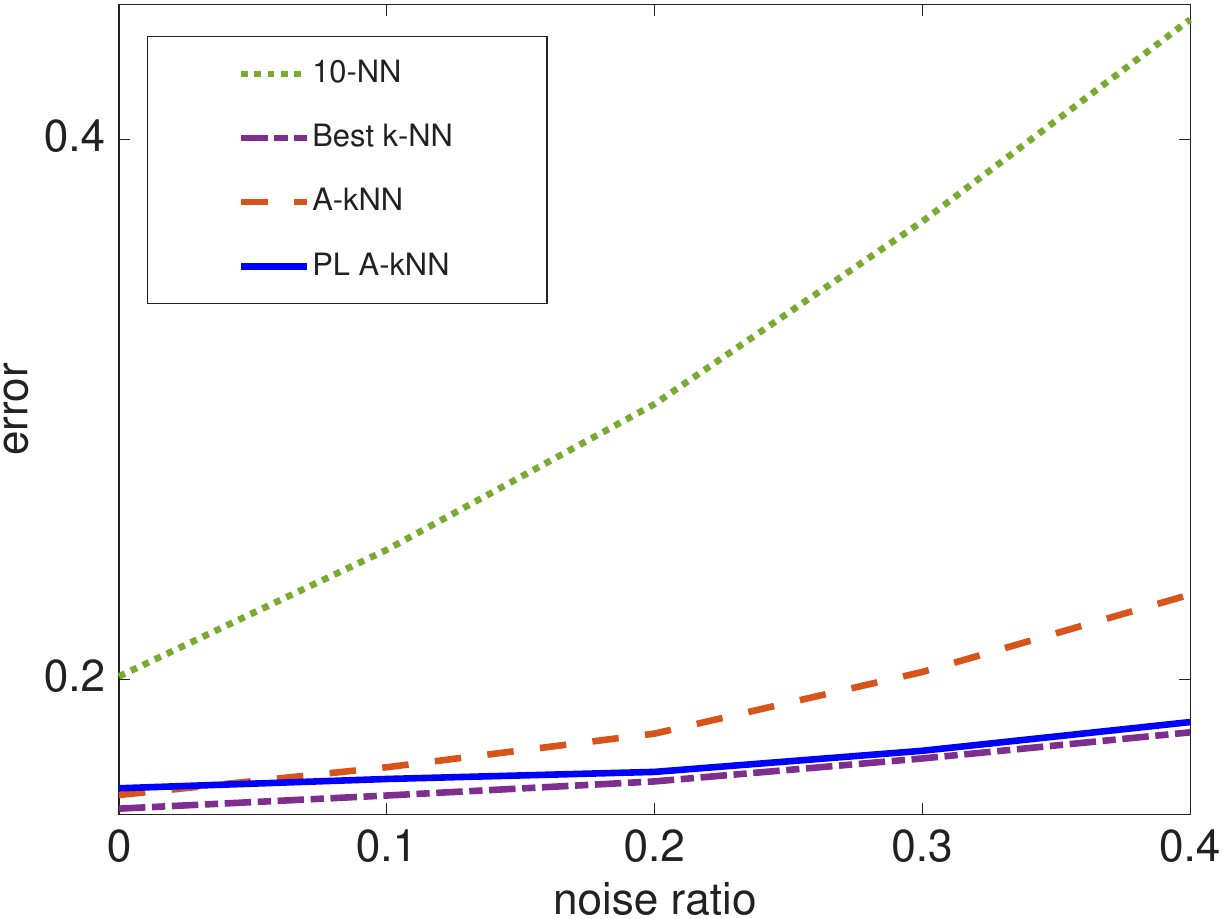}
        \caption{Fashion-MNIST}
    \end{subfigure}
    \hfill
    \begin{subfigure}[b]{0.48\textwidth}
        \centering
        \includegraphics[width=\linewidth]{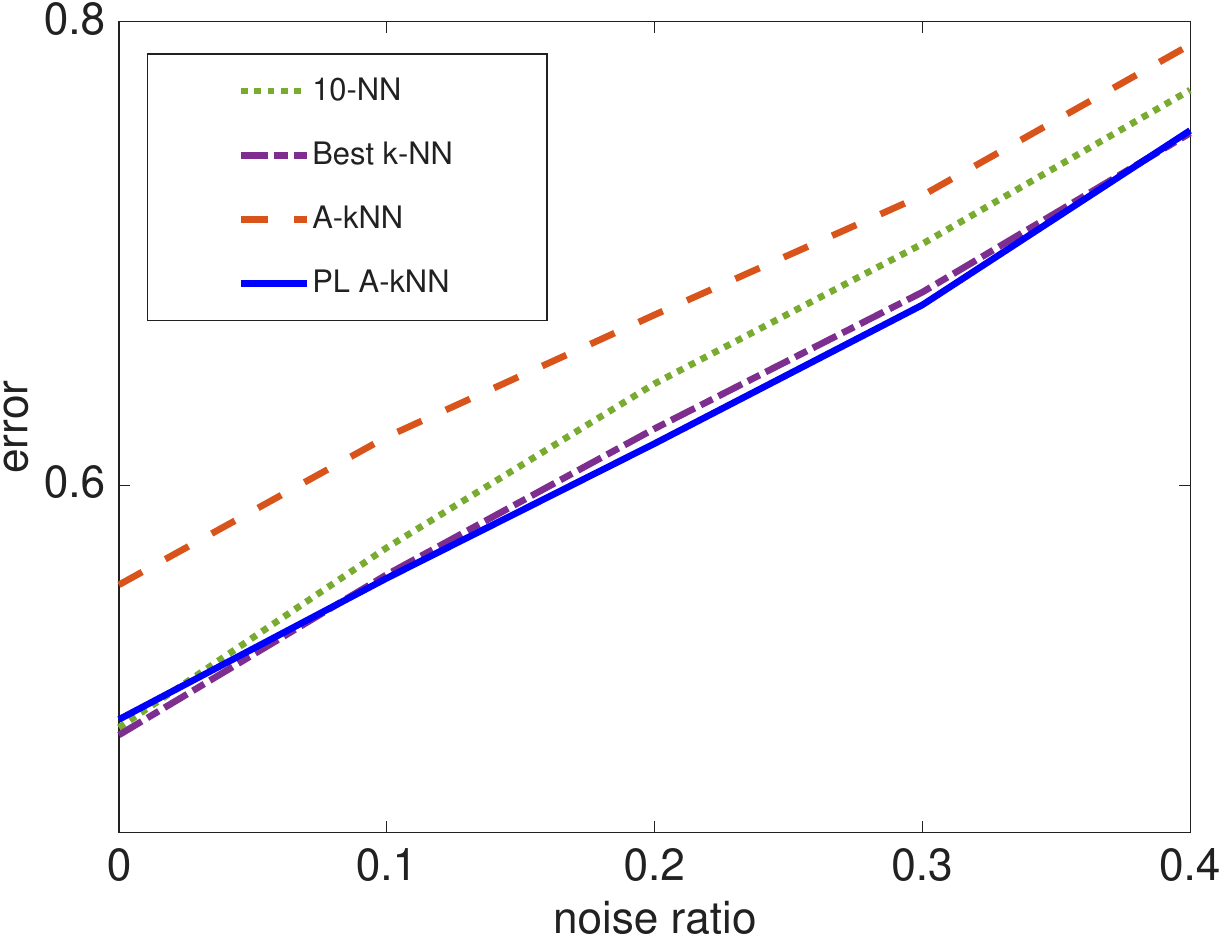}
        \caption{MSRCV2}
    \end{subfigure}
    \caption{Comparison of the error rates of PL A-$k$NN and state-of-the-art methods for  Fashion-MNIST and MSRCv2 under  an increasing noise rate. The results show that PL A-$k$NN  outperforms $10$-NN and A-$k$NN across a wide range of noise levels, while having comparable performance to the best $k$NN. See Appendix D for the comparison in MNIST CIFAR10 and MirFlickr.}
    \label{results_knn}
\end{figure*}

  The following result bounds the limiting risk of the PL A-$k$NN  algorithm when  the underlying distribution satisfies the label-aligned relaxation \eqref{relax}.

\begin{theorem} \label{limit_any_dist}
  (Limiting risk bounds) Let  $h_{n}$  be  a classifier from  PL A-$k$NN Algorithm \ref{alg1}  with parameters that satisfy   \eqref{conf_seq}. If the underlying probability distribution  satisfies the  relaxed label-aligned  condition \eqref{relax} for $G\subseteq \mathcal{X}$ and  $\theta$  , we  have  
  \begin{align*}
  \lim_{n\rightarrow \infty}\mathcal{R}(h_n)  \leq  \mathcal{R}^* + \theta P(G) 
  \end{align*}
  almost surely.
\end{theorem}
\begin{proof} See Appendix C.\end{proof}

The previous result shows  that the  PL A-$k$NN classifier provides reliable predictions  when the underlying probability distribution satisfies a  slight weakening of  the label-aligned condition \eqref{partial_condition}. In particular, the risk of PL A-kNN is close to the Bayes risk when only a small subset of instances fails to satisfy \eqref{partial_condition}, or when the most frequent labels over the bags  have conditional probabilities  similar to that of the most probable label.

The results above show that the presented PL A-$k$NN method provides strong performance guarantees with minimal assumptions. The next section  shows that the  algorithm is effective in practice under general PLL scenarios.

\section{Experimental results} \label{sec:experiments}
In this section we evaluate the performance of \hbox{PL A-$k$NN}  over a wide set of  artificial  partial label scenarios obtained from  the MNIST, Fashion-MNIST  and CIFAR10 datasets, as well as in  two real-life partial labels datasets, Mirflickr \cite{HuiskesLew2008} and MSRCv2 \cite{liu2012conditional}, under different levels of noise. For the vision datasets, we extract feature representations using GoogLeNet \cite{szegedy2015going} and subsequently apply principal component analysis to retain the 50 most informative components.

We   compare  PL  A-$k$NN  with four  other state of the art  partial labels methods with theoretical backing. In particular, we compare with  the CLPL method from  \cite{cour2011learning},  the LWS method from \cite{wen2021leveraged},  the PRODEN  method from \cite{lv2020progressive} and the   APL  method from \cite{lv2023robustness}. We use the linear variant of all models  to provide a fair comparison and   tune the models from previous work according to the specifications in their papers. 

We compare  PL  A-$k$NN with three $k$NN benchmarks:  the  ideal $k$NN that uses the optimal number of neighbors;  $10$NN,  which is the usual choice of $k$ in previous works; and the multiclass extension of A-$k$NN from \cite{balsubramani2019adaptive}.

\subsection{Generation of the artificial  partially labeled scenarios}
We construct partially labeled datasets from supervised datasets  by generating   bag generation processes  that vary  across instances  and allow noise. 

The bag generation process are generated as follows. We partition the training set into five clusters. For cluster $j$ and label $i$  we draw  uniformly  $\alpha_{i,j}\in [0, 0.8]$, which describe the  probability of adding  each incorrect label to the bag of  instances in $j$ with ground truth label $i$. The ground-truth label is, by construction, always included in the candidate set. This procedure generates bags with high  ambiguity and ensures that the bag generation process varies across clusters. To add  noise  we adopt the confusion process proposed by \cite{lv2023robustness}. With probability $\nu$, the ground truth label is replaced by a uniformly sampled random label. This corrupted label is then treated as the ground truth for bag generation---meaning it is guaranteed to appear in the bag, and other labels are added according to the same cluster-specific process described above. 

For the real partially labeled datasets,  label noise is introduced by removing the true label from the bags.

\subsection{ \texorpdfstring{Specifications for PL  A-$k$NN}{Specifications for PL  A-kNN} }
 We fix the hyperparameters of PL A-$k$NN as $c_1=0.5$ and $\delta=0.1$ across all experiments, with $T=400$ for vision datasets and $T=50$ for the  substantially smaller real-life datasets. Appendix D provides a detailed description of the preprocessing used for all nearest-neighbor methods.

\subsection{Results}

Figures \ref{results_sota} and \ref{results_knn} include the performance of PL A-$k$NN  with respect to the  other models over an increasing noise ratio. Specifically, Figure \ref{results_sota} compares PL A-$k$NN with four state-of-the-art methods for CIFAR-10 and MSRCv2, while Figure \ref{results_knn} compares it with three $k$NN-based benchmarks under the same datasets. We only include results for one synthetic partially labeled vision dataset and one real dataset because  they are representative of the trends observed across the remaining datasets. Complete results for all experimental settings are provided in Appendix~D.  All experiments use an 80/20 train-test split, and results are averaged over 100 independent runs.

Figure \ref{results_sota}   shows that PL A-$k$NN outperforms  state-of-the-art methods  in general scenarios under a wide range of noise levels. Unlike competing methods, \hbox{PL A-$k$NN} consitently delivers strong performance, either matching or exceeding the alternatives in every setting. By contrast, the other methods tend to perform well only under specific conditions but degrade significantly in others. For example, the LWS method \cite{wen2021leveraged} and PRODEN \cite{lv2020progressive} achieve high accuracy when noise is low but suffers substantial drops under high noise. Conversely, the CLPL method \cite{cour2011learning} performs very  poorly in low-noise scenarios but  improves  in  the highly noisy settings of the real partial labels datasets.

 Figure \ref{results_knn} shows that PL A-$k$NN provides a clear improvement over the standard $10$-NN baseline commonly adopted in this setting \citep{hullermeier2006learning,zhang2015solving}.  Notably, PL A-$k$NN  also outperforms the extension for A-$k$NN from \cite{balsubramani2019adaptive},  achieving a  comparable performance  to the  $k$NN under the best choice of number of neighbors.

The experimental results show that PL A-$k$NN can provide accurate classification rules across a wide range of  general scenarios—from cases with a high ratio of noise to cases where the true label is always included in the bags.

\section{Conclusion}
In this paper we  mathematically characterize the settings in which PLL is feasible. In addition, we present PL A-$k$NN, an  adaptive nearest neighbors algorithm for PLL  that is  effective in general scenarios. We provide theoretical guarantees for PL A-$k$NN under weaker assumptions than previous work. In particular, we prove the Bayes consistency of PL A-$k$NN for any label-aligned scenario and show that no other algorithm can achieve consistency in strictly more general scenarios than  PL A-$k$NN. Moreover, we also provide rates of convergence to the Bayes risk. Experimental results show that PL A-$k$NN can  outperform state-of-the-art methods in general scenarios,   matching the performance of the  ideal nearest neighbors with the optimal number of neighbors.

\subsection*{Limitations}
The proposed PL A-$k$NN method effectiveness depends on the availability of a preprocessing pipeline that yields a metric structure compatible with nearest-neighbor classification. When the features do not adequately capture label similarity, performance may degrade. Therefore, careful feature design and preprocessing remain important for practical deployment.

\subsection*{Acknowledgments}
Funding for this work was provided by the Spanish Ministry of Science and Innovation under grants PID2022-137063NB-I00 and PID2022-137442NB-I00, funded by MCIN/AEI/10.13039/501100011033 and European Union “NextGenerationEU”/PRTR; the BCAM Severo Ochoa accreditation CEX2021-001142-S (MICIN/AEI/10.13039/501100011033); the Basque Government BERC 2022–2025 IT2109-26 and ELKARTEK programs; and the U.S. National Science Foundation under grant CCF-2217058.
\bibliographystyle{plainnat}  
\bibliography{refs.bib}

\clearpage
\onecolumn
\appendix
\section{On the generality of the label-aligned condition}

In this appendix, we introduce the partial labels model of \citet{liu2014learnability}, which enables the recovery of the underlying distribution via empirical risk minimization, and  six other partial label models from the literature \citep{cour2011learning,cabannes2020structured,feng2020provably,wen2021leveraged,lv2023robustness} for which previous works have proposed Bayes consistent learning methods. We prove that    six of these models  correspond to  specific cases of   label-aligned   bags \eqref{partial_condition}. For the last partial labels model, we show that  it considers  scenarios   that   do not lead to label-aligned bags,  and that  there are  label-aligned cases  that  do not satisfy the model's assumptions,  proving  that neither condition is more general than the other.

\subsection{\texorpdfstring{Partial labels model from \cite{liu2014learnability}}{Partial labels model from liu2014learnability }}

The partial labels model from \cite{liu2014learnability} assumes deterministic labels  (i.e., probability distributions that assign all the probability  mass of each instance to a single label), that the ground-truth label is always included in the bags,  and that  no  other label  can always  co-occur with the ground-truth label  in the bags of an instance (see the beginning of Section 3 and Subsection 3.1 in \cite{liu2014learnability}). More precisely, the model in \cite{liu2014learnability} considers scenarios that satisfy
\begin{align*}
& P(y | x) \in \{0,1\}, && \forall\, y \in \mathcal{Y},\, x \in \mathcal{X}, \\[4pt]
& y \notin s \;\Longrightarrow\; P(s | y, x) = 0, && \forall\, y \in \mathcal{Y},\, x \in \mathcal{X}, \\[4pt]
& \text{for each } x \in \mathcal{X}, \text{ if } P(y=i | x)=1, 
   \text{ then } P(\mathcal{S}_j | x,i) < 1, 
   && \forall\, j \in \mathcal{Y} \setminus \{i\}.
\end{align*}
We now  prove  that the  PLL scenarios that satisfy the previous conditions are label-aligned. Let $P$ be a distribution that satisfies the previous assumptions. For a given instance $x$, let  $i= \argmax_{y \in \mathcal{Y}} P(y|x) $. Then,  as labels are deterministic, we have that  $P(i|x)=1$, and since the ground-truth  label is always included in the bags, $P(\mathcal{S}_i|x)=1$. The co-occurrence condition   implies that  $P(\mathcal{S}_j|x)=P(\mathcal{S}_j|i,x)<1, \forall j \neq i$. Therefore $i= \argmax_{y \in \mathcal{Y}} P(\mathcal{S}_y|x)$, and $P$  has  label-aligned bags.

\subsection{\texorpdfstring{Partial labels model from \cite{cour2011learning}}{Partial labels model from cour2011learning}}

The partial labels model from \citet{cour2011learning}  assumes  the  label-aligned condition  (see Proposition  5 in  \cite{cour2011learning}). Moreover, it also assumes that the bag generation process satisfies the following domination condition: 
\begin{align*}
    & \forall x \in \mathcal{X}, \text{ let } 
      i \in \argmax_{y \in \mathcal{Y}} P(\mathcal{S}_y|x), \quad 
      j \notin \argmax_{y \in \mathcal{Y}} P(\mathcal{S}_y|x),\\
    & \text{and } s \in \mathcal{S} \text{ such that } i, j \notin s, \text{ then}\\[4pt]
    & \quad P(s \cup \{i\}|x) > P(s \cup \{j\}|x) \,.
\end{align*}

\subsection{\texorpdfstring{Partial labels model from \cite{cabannes2020structured}}{Partial labels model from cabannes2020structured }}

The partial labels model from \cite{cabannes2020structured} assumes that, for each instance,  the intersection of all the possible bags of  that instance is exactly   the most probable label (see Definition 1 and 2, Proposition 1 in \cite{cabannes2020structured}). More precisely,  the model in \cite{cabannes2020structured} considers scenarios that satisfy
\begin{align*}
  & \forall x \in \mathcal{X},\text{let}  \hspace{2mm}  i=\argmax_{y \in \mathcal{Y}} P(y|x),  \hspace{2mm}  \text{then} \\
   &
   \bigcap_{s\in \mathcal{S}:P(s|x)>0} s= \{i\}.
\end{align*}

We now  prove  that the  PLL scenarios that satisfy the previous conditions are label-aligned. Let $P$ be a distribution that satisfies the previous assumptions. For a given instance $x$, let  $i= \argmax_{y \in \mathcal{Y}} P(y|x) $. Then, we have that  $P(\mathcal{S}_i|x)=1$ and  $P(\mathcal{S}_j|x)<1, \forall j \neq i$, since only the most probable label is always included in the bags. Therefore, $i= \argmax_{y \in \mathcal{Y}} P(\mathcal{S}_y|x)$  and $P$  has  label-aligned bags. 

\subsection{\texorpdfstring{Partial labels model from \cite{feng2020provably}}{Partial labels model from feng2020provably}}

The partial labels model from \cite{feng2020provably} assumes that all bags that contain the  ground-truth label are equally probable (see Equation 5 in \cite{feng2020provably}). More precisely,  the model in \cite{feng2020provably} considers scenarios that satisfy
\begin{align*}
    P(s|y,x)=\begin{dcases}
        \frac{1}{2^{|\mathcal{Y}|-1}-1} & y\in s\\
        0 &y\notin s.
        \end{dcases}
\end{align*}

We now  prove  that all the  PLL scenarios that satisfy the previous condition have label-aligned bags. Let $P$ be a distribution that satisfies the previous assumption. 

Let  labels $i,j,k,k_2 \in \mathcal{Y}$ be such that  $i \neq k$ and $j \neq k_2$. Then, 
\begin{align*}
&P(\mathcal{S}_i|k,x)= \sum_{s\in \mathcal{S}_i: k\in s} P(s|k,x) =\sum_{s\in \mathcal{S}_i: k\in s}  \frac{1}{2^{|\mathcal{Y}|-1}-1}=|{s\in \mathcal{S}_i: k\in s}|  \frac{1}{2^{|\mathcal{Y}|-1}-1}=\\
&|{s\in \mathcal{S}_j: k_2\in s}|  \frac{1}{2^{|\mathcal{Y}|-1}-1}= \sum_{s\in \mathcal{S}_j: k_2\in s}  \frac{1}{2^{|\mathcal{Y}|-1}-1} =\sum_{s\in \mathcal{S}_j: k_2\in s} P(s|k_2,x) =P(\mathcal{S}_j|k_2,x).
\end{align*}
We can prove analogously that $P(\mathcal{S}_i|i,x)=P(\mathcal{S}_j|j,x)$.
 
 We  denote now   $\alpha(x)=P(\mathcal{S}_i|i,x) -P(\mathcal{S}_j|k,x)$, where  $j \neq k$.  
 Due to the previous equalities, the value of alpha 
 does not depend on the labels, and  is clear that $\alpha(x)>0$  $\forall x \in \mathcal{X}$, since $|\mathcal{S}_i|>|{s\in \mathcal{S}_j: k\in s}|$.

We then have that 
$P(\mathcal{S}_i|x)-P(\mathcal{S}_j|x)= \sum_{y \in \mathcal{Y}} (P(\mathcal{S}_i|y,x)-P(\mathcal{S}_j|y,x))P(y|x)=\alpha(x)(P(i|x)-P(j|x))$. Therefore,  for any given instance, if $P(i|x)>P(j|x)$, then $P(\mathcal{S}_i|x)>P(\mathcal{S}_j|x)$, so $P$ has label-aligned bags.

\subsection{\texorpdfstring{Partial labels model from \cite{wen2021leveraged}}{Partial labels model from wen2021leveraged}}

The partial labels model from  \cite{wen2021leveraged} assumes deterministic labels, that the ground-truth label $i$ is always included in the bags,  and  that the other labels $j$ in the bags are independently drawn given  probabilities  $q_{ij}<1$  (see Subsection 3.3.1 and Equation 11 in \cite{wen2021leveraged}). More precisely, the model in  \cite{wen2021leveraged} considers scenarios that satisfy 
\begin{align*}
    & P(y|x) \in \{0,1\}, 
      & \forall y \in \mathcal{Y}, x \in \mathcal{X},\\
    & P(s|i,x) = \prod_{k \in s} q_{i,k} \prod_{j \notin s} (1 - q_{i,j}), 
      & \forall s \in \mathcal{S}, \forall y \in \mathcal{Y}, x \in \mathcal{X}, 
      \text{ where}\\
    & q_{i,i} = 1, \quad 
       \text{and} \quad q_{i,j} < 1, \quad i \neq j.
\end{align*}

We prove now that all distributions that satisfy the previous conditions are label-aligned. Let $P$ be a distribution that satisfies the previous assumptions. For a given instance $x$, let  $i\in \mathcal{Y}$ be the label with $P(i|x)=1$. Then $P(\mathcal{S}_i|x)=P(\mathcal{S}_i|i,x)=1$, since  for any bag $s$ with $i\notin s$ we have that $P(s|i,x)=0$. Moreover, $P(\{i\}|x)=P(\{i\}|i,x)>0$, since $q_{i,j}<1 $ $ \forall i\neq j$, and therefore $P(\mathcal{S}_j|x)<1$ for any  $j \neq i$. We conclude then that $P$ is label-aligned.

\textbf{Note:} \cite{wen2021leveraged} states that their model is intended for stochastic labels, in which case the bags are not always label-aligned.
 However, if  labels are  stochastic, their assumptions on the bag generation process are incompatible with the  Bayes consistency of any algorithm. In  fact, the proofs of their theorems  assume implicitly deterministic labels.  The following provides a short counterexample for the stochastic case, where we show that no algorithm can be Bayes consistent under the assumptions of the model.

 Let us consider a binary label space and an instance space consisting of a single point. Let two partial label probability distributions be given by
\begin{align*}
& P(1) = 2/3 , \hspace{2mm} P(2) = 1/3 , \hspace{2mm} q_{1,1} = 1 , \hspace{2mm} q_{1,2} = 2/3 , \hspace{2mm} q_{2,2} = 1 , \hspace{2mm} q_{2,1} = 0 \\
& \hat{P}(1) = 1/3 , \hspace{2mm} \hat{P}(2) = 2/3 , \hspace{2mm} \hat{q}_{1,1} = 1 , \hspace{2mm} \hat{q}_{1,2} = 1/3 , \hspace{2mm} \hat{q}_{2,2} = 1 , \hspace{2mm} \hat{q}_{2,1} = 1/2 .
\end{align*}
Both scenarios abide the conditions from the partial labels model in \cite{wen2021leveraged} and  generate the same distribution of bags 
\begin{align*}
P(\{1\}) = 2/9 , \hspace{2mm} P(\{2\}) = 3/9 , \hspace{2mm} P(\{1,2\}) = 4/9 ,
\end{align*}
while having the two  different labels as  Bayes rules. Therefore, it is not   possible for an algorithm to be Bayes consistent for  all   scenarios  considered in  the model when the  labels are stochastic.

\subsection{\texorpdfstring{ Noiseless partial labels model from \cite{lv2023robustness}}{Noiseless partial labels model from v2023robustness}}

The  noiseless partial labels model from \cite{lv2023robustness} assumes  deterministic labels, that the ground truth label is always included in the bags,    that the  bag generation process   is the same for all instances, and that  the frequencies within the bags   of labels other than the ground-truth label are strictly less than one (see Section 4.2 and Theorem 1  in \cite{lv2023robustness}).  More precisely,  the noiseless model in \cite{lv2023robustness} considers scenarios that satisfy
\begin{align*}
& P(y|x) \in \{0,1\} & \forall y \in \mathcal{Y},\forall  x \in \mathcal{X} \\    
 & y \notin s \Longrightarrow P(s|y,x)=0 & \forall y \in \mathcal{Y}, \forall x \in \mathcal{X} \\
 & P(s|y,x_1)=P(s|y,x_2)    & \forall y \in \mathcal{Y}, \forall x_1,x_2 \in \mathcal{X}\\
 &P(\mathcal{S}_j|i,x)<1 & i\ne j.
\end{align*}

We now  prove that all PLL scenarios  that satisfy the previous assumptions have  label-aligned bags . Let $P$ be a distribution that satisfies the previous assumptions. For a given instance, let  $i\in \mathcal{Y}$ be the label with $P(i|x)=1$. We then  have that  $P(\mathcal{S}_i|x)=P(\mathcal{S}_i|i,x)=1$, since  for any bag $s$ with $i\notin s$ we have that $P(s|i,x)=0$. On the other hand  $P(\mathcal{S}_j|x)=P(\mathcal{S}_j|i,x)<1$ for any  $j \neq i$. We conclude then that $P$ has  label aligned bags.

\subsection{\texorpdfstring{ Noisy partial labels model from \cite{lv2023robustness}}{Noisy partial labels model from v2023robustness}}

The  noisy  partial labels model from \cite{lv2023robustness} assumes deterministic labels,  that  the bags generation process does not vary across instances, and that the most probable label is the one most likely to be selected when randomly choosing a label from the bags(see Section 4.2 and Theorem 5  in \cite{lv2023robustness}). More precisely,  the noisy model in \cite{lv2023robustness} considers scenarios that satisfy
\begin{align*}
    & P(y|x) \in \{0,1\}, 
      \quad \forall y \in \mathcal{Y}, \forall x \in \mathcal{X},\\
    & P(s|y,x_1) = P(s|y,x_2), 
      \quad \forall y \in \mathcal{Y}, \forall x_1,x_2 \in \mathcal{X},\\
    &   \sum_{s \in \mathcal{S}_i} \frac{1}{|s|} P(s|i,x) 
      > \sum_{s \in \mathcal{S}_j} \frac{1}{|s|} P(s|i,x) \quad  \forall i \in  \mathcal{Y}, \forall j \neq i, \forall x \in \mathcal{X} .
\end{align*}

In this case,  there are PLL scenarios  that  satisfy the model's assumptions but do not  have label-aligned bags, and  there are  label-aligned cases that  do not satisfy the model's assumptions.  For example,  in a  partial labels problem  with three labels and only one instance, let the distribution over the labels  of $P_1$ be
\begin{align*}
 P_1(1)=1    \hspace{2mm} P_1(2)=0 \hspace{2mm} P_1(3)=0 
\end{align*}
and the bag generation process when $1$ is the true label be 
\begin{align*}
    &P_1(\{1\}|1)=0.1 \hspace{5.5mm} P_1(\{2\}|1)=0\hspace{5.5mm}  P_1(\{3\}|1)=0.4 \\
    & P_1(\{1,2 \}|1)=0.5 \hspace{2mm} P_1(\{2,3\}|1)=0 \hspace{2mm}  P_1(\{1,3\}|1)=0\\
    &  P_1(\{1,2,3\}|1)=0. 
\end{align*}
 We then have that the distribution over the bags is  $P_1(s)=P_1(s|1)$, so $P_1(\mathcal{S}_1)=0.6 $, $P_1(\mathcal{S}_2)=0.5 $, and $P_1(\mathcal{S}_3)=0.4 $, and therefore  $P_1$ has  label-aligned bags . However,
\begin{align*}
 \sum_{s \in \mathcal{S}_1}\frac{1}{|s|}P_1(s|1,x)=0.35< 0.4=\sum_{s \in \mathcal{S}_3} \frac{1}{|s|}P_1(s|1,x)    
\end{align*}
and therefore  $P_1$ does not satisfy the previous assumptions.

On the other hand, let  $P_2$ be the underlying distribution of other partial labels  problem   with three labels and only one instance, with distribution over the labels  
\begin{align*}
 P_2(1)=0    \hspace{2mm} P_2(2)=0 \hspace{2mm} P_2(3)=1,
\end{align*}
 
and  bag generation process when $3$ is the ground-truth label 
\begin{align*}
    &P_2(\{1\}|3)=0.1 \hspace{5.5mm} P_2(\{2\}|3)=0\hspace{5.5mm}  P_2(\{3\}|3)=0.4 \\
    & P_2(\{1,2 \}|3)=0.5 \hspace{2mm} P_2(\{2,3\}|3)=0 \hspace{2mm}  P_2(\{1,3\}|3)=0\\
    &  P_2(\{1,2,3\}|3)=0. 
\end{align*}
 We then  have that the distribution over  the bags is  $P_2(s)=P_2(s|3)$, so $P_2(\mathcal{S}_1)=0.6 $, $P_2(\mathcal{S}_2)=0.5 $, and $P_2(\mathcal{S}_3)=0.4 $ and therefore  $P_2$ does not have  label-aligned bags. However,

\begin{align*}
 \sum_{s \in \mathcal{S}_1}\frac{1}{|s|}P_2(s|3,x)=0.35 \hspace{2mm}
 \sum_{s \in \mathcal{S}_1}\frac{1}{|s|}P_2(s|3,x)=0.25  \hspace{2mm}
 \sum_{s \in \mathcal{S}_3} \frac{1}{|s|}P_2(s|3,x)= 0.4 ,   
\end{align*}
 so $P_2$ satisfies 
the assumptions of the model.

\section{\texorpdfstring{Extended Theoretical Analysis of PL A-$k$NN}{Extended Theoretical Analysis of PL A-$k$NN}}

\subsection{Notation}
Let  $\{ (x_l, s_{l})\}_{l=1}^n $  be the set of  training examples. For any subset $G\subseteq \mathcal{X}$, the empirical count and mass are taken as:
\begin{align*}
   & \#_n(G)= |\{ l: x_l\in G\}| \\
   & P_n(G)= \frac{ \#_n(G)}{n}.
\end{align*}

In addition, for any subset  $G\subseteq \mathcal{X}$ with    non zero empirical mass, the empirical frequencies of the labels  in the bags of $G$  are  by definition:

\begin{align*}
 &P_n(\mathcal{S}_y|G)= \frac{\sum_{l=1}^n \mathbb{I}( y\in s_l)\mathbb{I}( x_l\in G)}{\#_n(G)}   & \forall y \in \mathcal{Y}.
\end{align*}
\subsection{\texorpdfstring{Fundamental differences between PL A-$k$NN and  A-$k$NN}{Fundamental differences between PL A-$k$NN and  A-$k$NN}}

PL A-$k$NN is inspired by the adaptive $k$-nearest neighbor  (A-$k$NN) method  for binary classification from  \cite{balsubramani2019adaptive}.  PLL requires to consider multiclass settings, and the proposed PL A-kNN differs fundamentally from the multiclass  supervised classification extension suggested in \cite{balsubramani2019adaptive}. Such an  extension increases the neighborhood until one label satisfies
\begin{align*} 
    P_n(y\,|\, B_k(x)) - \frac{1}{|\mathcal{Y}|} \geq \Delta(n,k,\delta),
\end{align*}
where $B_k(x)$ represents the ball with exactly the $k$ nearest neighbors of $x$, and then predicts that label. In other words,  the A-$k$NN extension increases the neighborhood until the empirically most frequent label  has higher frequency than  threshold $\Delta$, thereby  only considering  the frequency of that  label. This criterion naturally extends to the partial-labels case as follows:

\begin{align*} 
    P_n(\mathcal{S}_y\,|\, B_k(x)) - \frac{1}{|\mathcal{Y}|} \geq \Delta(n,k,\delta),
\end{align*}

In contrast, PL A-$k$NN progressively eliminates labels $i$ whose frequency deviates from the maximum frequency  by more than $\Delta$, that is,
\begin{align*}
    \max_{ y \in \hat{s}} P_n(\mathcal{S}_y\,|\, B_k(x)) - P_n(\mathcal{S}_i\,|\, B_k(x)) \geq \Delta(n,k,\delta),
\end{align*}

gradually narrowing the candidate set while enlarging the neighborhood. This margin-based elimination considers all  of the  label frequencies, not just the largest one.

To illustrate the  differences between the two methods, consider two multiclass classification problems (which can be viewed as trivial partial-label problems where each bag contains only the ground-truth label) with a single instance, three labels, and the following label distributions: in the first case, $Q_1 = (0.50, 0.49, 0.01)$; in the second case, $Q_2 = (0.40, 0.30, 0.30)$. These vectors represent both the label distributions and the frequencies of the labels within the bags.
 For A-$k$NN, which only checks frequency of  the  empyrically most frequent label, the second case may require more neighbors to satisfy the threshold, even though the first case is harder due to the nearly tied most probable labels. PL A-$k$NN, by considering frequency differences, naturally requires more neighbors in the first case 
and fewer in the second case. This shows that PL A-$k$NN increases the neighborhood size appropriately for harder cases, whereas A-$k$NN can use too many or too few neighbors, because it only considers the frequency  of  the most frequent label and disregards the rest of the information.

\subsection{Finite sample convergence rates }
To establish finite-sample convergence rates, the data must satisfy certain smoothness conditions.
 Therefore, we assume   two smoothness conditions over the underlying probability distribution $P\in \Delta(\mathcal{X}  \times \mathcal{Y} \times  \mathcal{S})$ for the next  result.

First, we  adapt the Tsybakov-margin  for binary distributions \citep{audibert2007fast,mammen1999smooth,tsybakov2004optimal}   to the frequencies of the labels over the  bags. For any $\beta> 0$, we say $P$ satisfies the $\beta$-margin condition if there exists a constant $C_2$ > 0 such that
\begin{align*}
     &P(\{ x\in \mathcal{X}: \forall i \in \argmax_{y\in \mathcal{Y}} P(\mathcal{S}_y|x) \hspace{2mm} \exists  j \notin \argmax_{y\in \mathcal{Y}}  P(\mathcal{S}_y|x) \hspace{2mm}   P(\mathcal{S}_i|x)- P(\mathcal{S}_j|x) \leq t\}) \leq C_2\cdot t^\beta
     &\forall t \geq 0.
     \end{align*}
 We essentially require the  probability  mass of the instances  with  a frequency  margin  smaller than $t$ to be gracefully  bounded by  a power  function of $t$. Larger  values of $\beta$ imply smoother distributions.
 
 For the second  smoothness condition we adapt  to PLL a generalization of the
usual  $\alpha$-Holder condition \citep{chaudhuri2014rates}, which is a common assumption over $P$ for nonparametric estimators. For $\alpha>0, L>0$ we say that  $P$ is $(\alpha,L)$-smooth in the finite-dimensional normed  space $(\mathcal{X},P)$ if for all $x\in \mathcal{X}$ 
\begin{align*}
   & |P(\mathcal{S}_y|B(x,r))-P(\mathcal{S}_y|x)| \leq L P(B(x,r))^\alpha & \forall y \in \mathcal{Y}, \forall r \geq 0.
\end{align*}
Thus, we require the drift of the frequencies of the labels over  the bags in   balls to be bounded  by $\alpha$-exponential  values  of  the mass of the ball. Larger  values of $\alpha$ indicate that  $P$ behaves smoother.

The following theorem provides rates of convergence   that depend explicitly on the number of samples.
\begin{theorem}  \label{explicit_rates}
  (Explicit rates of convergence)   Let $h$  be   the classifier from 
PL A-$k$NN Algorithm \ref{alg1}  with   maximum number of iterations $T$ and   confidence parameter $\delta $ satisfying \eqref{conf_seq}.
Let the underlying probability distribution satisfy the label-aligned condition \eqref{partial_condition}, be $(\alpha, L)$-smooth, and satisfy the $\beta$-margin condition for $\alpha > 0$ and $\beta > 0$. Then there exists a constant $C_3$ such that 
   \begin{align*}
      & R(h)-R^* \leq \delta + C_3 \left( \frac{1}{n}\max \Bigl\{\log(n),\log(|\mathcal{Y}|/\delta)\Bigl\}\right) ^{ \frac{\beta \alpha}{2\alpha+1}}
         \end{align*}
 holds with probability at least $1-\delta$.
\end{theorem}

\textit{ \textbf{Proof}: See Appendix C.}

 As  expected, larger values of  $\alpha$ and  $\beta$ provide faster rates of convergence. 
\section{Theorems, Lemmas, and Proofs}

\subsection{Notation and definitions}
Let  $\{ (x_l, s_{l})\}_{l=1}^n $  be the set of  training examples. For any subset $G\subseteq \mathcal{X}$, the empirical count and mass are taken as:
\begin{align*}
   & \#_n(G)= |\{ l: x_l\in G \}| \\
   & P_n(G)= \frac{ \#_n(G)}{n}.
\end{align*}
In addition, for any subset  $G\subseteq \mathcal{X}$ with    non zero empirical mass, the empirical frequencies of the labels  in the bags of $G$  are  by definition:
\begin{align*}
 &P_n(\mathcal{S}_y|G)= \frac{\sum_{l=1}^n \mathbb{I}( y\in s_l)\mathbb{I}( x_l\in G)}{\#_n(G)}   & \forall y \in \mathcal{Y}.
\end{align*}

\begin{definition}
The \emph{support} of the distribution \(P\), denoted by \(\mathrm{supp}(P)\), is the set
\begin{align} \label{support}
\mathrm{supp}(P) = \{ x \in \mathcal{X} \mid \forall \, r > 0, \; P(B(x,r)) > 0 \},
\end{align}
where \(B(x,r)\) is the ball of radius \(r\) centered at \(x\). 
\end{definition}

We proceed with a smoothness condition that holds for finite dimensional normed spaces.
\begin{definition}(Lebesgue differentiation condition) \label{def_lebesgue}
Let $(\mathcal{X},d,P)$ be a metric measure space. We say that  $(\mathcal{X},d,P)$ satisfies the Lebesgue differentiation condition if for any bounded measurable $f: \mathcal{X} \longrightarrow \R$ and for almost all ($P$-a.e.) $x \in \mathcal{X}$, we have

\begin{align} \label{lebesque}
    \lim_{r\rightarrow 0} \frac{1}{P(B(x,r))}\int_{B(x,r)} f \hspace{1mm}dP = f(x).
\end{align}
\end{definition}

\subsection{Proof of Theorem \ref{bag_aligned}}
\begin{proof}
Given an instance $x \in \mathcal{X}$, let  $Q_1,Q_2 \in \Delta( \mathcal{Y})$ be  two distributions over the labels  such that  $\sum_{y\in \mathcal{Y}} P(s|y,x)Q_1(y)=\sum_{y\in \mathcal{Y}} P(s|y,x)Q_2(y) \hspace{2mm} \forall s \in \mathcal{S}$. Since $P(s|y,x)$ is label-aligned, we have that 
\begin{align*}
& \argmax_{y \in \mathcal{Y}} Q_1(y)= \argmax_{y \in \mathcal{Y}} \sum_{s\in \mathcal{S}_y}\sum_{i\in \mathcal{Y}} P(s|i,x)Q_1(i)=\\
&\argmax_{y \in \mathcal{Y}} \sum_{s\in \mathcal{S}_y}\sum_{i\in \mathcal{Y}} P(s|i,x)Q_2(i)= \argmax_{y \in \mathcal{Y}} Q_2(y)
\end{align*}
and therefore the proof is concluded.
\end{proof}

\subsection{Proof of Theorem \ref{th1}}
Before providing the proof of the theorem we state and prove some technical lemmas. We first state two results from previous work that are needed for our main result
\begin{lemma}(\cite{chaudhuri2010rates} Lemma 7) \label{lemma7}
There is a universal constant $c_0$ such that the following holds. Let $\mathcal{B}$ be any class of measurable subsets of $\mathcal{X}$ of VC dimension $d_0$. Pick any $0 < \delta < 1$. Then with probability at least $1 - \delta^2/2$ over the choice of $x_1,x_2 \ldots, x_n$, for all $B \in \mathcal{B}$ and for any integer $k$, we have
\begin{align*}
P(B) \ge \frac{k}{n} + \frac{c_0}{n} \max\left(k, d_0 \log\frac{n}{\delta}\right) \quad \Rightarrow \quad P_n(B) \ge \frac{k}{n}.
\end{align*}
\end{lemma}

\begin{theorem}(\cite{balsubramani2019adaptive} Theorem 8) \label{theorem 8}
Let $P$ be a probability distribution over $\mathcal{X}$, and let $\mathcal{A}, \mathcal{B}$ be two families of measurable subsets of $\mathcal{X}$ such that $\operatorname{VC}(\mathcal{A}), \operatorname{VC}(\mathcal{B}) \le d_0$. Let $n \in \mathbb{N}$, and let $x_1, \ldots, x_n$ be $n$ i.i.d. samples from $P$. Then, the following event occurs with probability at least $1 - \delta$:
\[
\forall A \in \mathcal{A},\ \forall B \in \mathcal{B} :\quad
\left| P(A |B) - P_n(A |B) \right| \le \sqrt{ \frac{k_0}{\#_n(B)} },
\]
where $k_0 = 1000 \left( d_0 \log(8n) + \log\left( \frac{4}{\delta} \right) \right)$.
\end{theorem}

We now state and prove two technical lemmas that are needed for the main proof of the theorem.

\begin{lemma} \label{lemma2}
    Let $P$ be a probability distribution over $\mathcal{X}\times \mathcal{Y} \times \mathcal{S}$. There is a universal constant $c_1>0$ such that, for each $x \in \mathcal{X}$, the following events occur with  probability at least $1-\delta^2/2$:
\begin{align} \label{prelema}
        &|P(\mathcal{S}_y|B(x,r))-P_n(\mathcal{S}_y|B(x,r))| \leq  \frac{c_1}{2}  \sqrt{\frac{ \log(n) +\log(|\mathcal{Y}|/\delta)}{ \#_n(B(x,r))}} &\forall r>0,\forall y\in \mathcal{Y} . 
    \end{align}
Let $d_0$ be the VC dimension of   the set of balls in $\mathcal{X}$. Then, we have that the following events occur with probability at least $1-\delta^2/2$:
\begin{align}
        &|P(\mathcal{S}_y|B(x,r))-P_n(\mathcal{S}_y|B(x,r))| \leq \frac{c_1}{2} \sqrt{\frac{ d_0 \log(n) +\log(|\mathcal{Y}|/\delta)}{ \#_n(B(x,r))}} &\forall r>0 ,\forall y\in \mathcal{Y}  , \forall x  \in \mathcal{X}.
\end{align}
\end{lemma}

\begin{proof}
Let $x\in \mathcal{X}$, $y \in \mathcal{Y}$, and   $\mathcal{B}_x$ be the set of all balls that  are centered on $x$, which has VC dimension 1.
We define the two following sets of sets:
\begin{align*}
 & \mathbf{A}_x^y=\{B\times \mathcal{S}_y: B\in \mathcal{B}_x\}  \hspace{3mm} \\  
     &\mathbf{B}_x=\{ B\times 2^\mathcal{Y}: B \in \mathcal{B}_x \}.   
\end{align*}

 It is easy to see then than $VC(\mathbf{A}_x^y)=1$ and $VC(\mathbf{B}_x)=1$, as  both \( \mathbf{A}_x^y \) and \( \mathbf{B}_x \) are obtained by taking Cartesian products of sets in \( \mathcal{B}_x \) (which has VC dimension 1) with fixed sets.

Therefore, taking $c_1>2\sqrt{8000\log(8)}$ and by Theorem \ref{theorem 8} we have proven that with  probability at least $1-\delta^2/(2|\mathcal{Y}|)$.
\begin{align*}
        &|P(\mathcal{S}_y|B(x,r))-P_n(\mathcal{S}_y|B(x,r))| \leq  \frac{c_1}{2}  \sqrt{\frac{ \log(n) +\log(|\mathcal{Y}|/\delta)}{ \#_n(B(x,r))}} &\forall r>0.
    \end{align*}
Then, the first part of the lemma in \eqref{prelema} follows by applying the union bound  for each label.

Suppose now that $d_0$ is  the VC dimension of   the set of balls $\mathcal{B}$ in $\mathcal{X}$ and let $y \in \mathcal{Y}$. We define  the following sets of sets
\begin{align*}
 & \mathbf{A}^y =\{B\times \mathcal{S}_y: B\in \mathcal{B}\}  \hspace{3mm} \\  
     &\mathbf{B}=\{ B\times 2^\mathcal{Y}: B \in \mathcal{B} \}.   
\end{align*}

We  then have,  by following the same reasoning as before, that $VC( \mathbf{A})= d_0$ and $VC(\mathbf{B})=d_0$.

Therefore, taking $c_1>2\sqrt{4000(d_0+1)\log(8)}$ and proceeding analogously to the first part of this proof we   have proven  the second part of the lemma.
\end{proof}

\begin{lemma} The following set of points \label{lemma1}
\begin{align*}
    \bigl\{  x\in \mathcal{X}: \hspace{1mm}  \adv(x)=0   \bigl\}
\end{align*}
has zero $P$-measure.
\end{lemma}

\begin{proof}

The instance space ($\mathcal{X},d)$ is a finite dimensional normed  space. Therefore, the Lebesgue differentiation condition \ref{def_lebesgue} holds \citep{heinonen2001lectures}. Let $ \mathcal{X}'$ be the subset of  the support  of distribution $P$ \eqref{support}  for which condition  \eqref{lebesque} is satisfied for the $|\mathcal{Y}|$ functions $P(\mathcal{S}_y| x)$. We have $P(\mathcal{X}')=1$,  since the Lebesgue differentiation condition holds and the support of $P$ has measure $1$. If we see that all the elements in  $ \mathcal{X}'$  have positive advantage,  then the proof is concluded.

Let $x$ be an instance in $\mathcal{X}'$. If  $\argmax_{y\in \mathcal{Y}} P(\mathcal{S}_y|x)=\mathcal{Y}$ then  by definition $\adv(x)=1$. For other cases, since the point $x$ satisfies \eqref{lebesque} for all  $P(\mathcal{S}_y|x) $ functions  we have that  $\forall i  \in \argmax_{y\in \mathcal{Y}} P(\mathcal{S}_y|x) $ and \hbox{$ \forall j \in \mathcal{Y} \setminus \argmax_{y\in \mathcal{Y}} P(\mathcal{S}_y|x)  $} exists $r_{i,j}>0$ such that
 \begin{align*}
        P(\mathcal{S}_i|B(x,r))- P(\mathcal{S}_j|B(x,r))\geq \frac{ P(\mathcal{S}_i|x)- P(\mathcal{S}_j|x)}{2} \hspace{3mm}   0 <r < r_{i,j}.
    \end{align*}
Therefore, the result is obtained, because $x$ is $(p,\gamma)$-salient for $p=\min\{A,P(B(x,\min r_{i,j}))\}>0$ and $\gamma=min_{j \notin \argmax_{y\in \mathcal{Y}}P(\mathcal{S}_y|x) }\frac{P(\mathcal{S}_i|x)- P(\mathcal{S}_j|x)}{2}$, having positive advantage in the region $B(x,r_A(x))$.     
\end{proof}

\begin{proof}[Proof of Theorem \ref{th1}]
    Let  $c_0$ and $c_1$ be  the constants of Lemma \ref{lemma7} and Lemma \ref{lemma2}. We define $c_2=\max(c_1,1/4)\sqrt{1+c_0}$ and take $C=16c^2_2$.

Let $x\in \mathcal{X}$  such that $\adv(x)>0$, and $\mathcal{B}$ the set of all balls that are centered on $x$. Following Lemma \ref{lemma7} and Lemma \ref{lemma2} we have  that with probability at least $1-\delta^2$ the following two properties hold for all $B\in \mathcal{B}$:

\begin{itemize}
    \item For any integer $k$ we have $\#_n(B) \geq k$ whenever $nP(B)\geq k+c_0 \max\{k,\log(n/\delta)\}$.
    \item $ |P_n(\mathcal{S}_y|B)-P(\mathcal{S}_y|B))| \leq \frac{1}{2} \Delta(n,\#_n(B),\delta) ,\hspace{3mm} \forall y\in \mathcal{Y}$, \hspace{3mm} and therefore  \\
    $|(P_n(\mathcal{S}_i|B)-P_n(\mathcal{S}_j|B))-(P(\mathcal{S}_i|B)-P(\mathcal{S}_j|B))| \leq \Delta(n,\#_n(B),\delta) , \forall i,j\in \mathcal{Y}$.

\end{itemize}

Assume henceforth that the above two conditions hold. If $\argmax_{y\in \mathcal{Y}} P(\mathcal{S}_y|x) =\mathcal{Y}$,  we  also have  that $\argmax_{y\in \mathcal{Y}}P(y|x) =\mathcal{Y}$, since   the bags are label-aligned, and therefore $h(x)=h^*(x)$.  Otherwise, let $y_0 \notin \argmax_{y\in \mathcal{Y}}  P(\mathcal{S}_y|x) $. If we show that $h(x) \neq y_0$, the proof is concluded, as we  show that $h(x)$ is not any of the suboptimal labels, which is equivalent to showing that  $h(x)$  is one of the most frequent labels in the bags of $x$.

By the definition of advantage, point $x$ is $(p_,\gamma)$-salient for  some $p\gamma>0$ with $\adv(x)= p\gamma^2$.

   The criterion   (\ref{avd_cond}) in the theorem statement implies that:
\begin{align}
       \gamma   \geq 2c_2\sqrt{\frac{\log(n)+\log(|\mathcal{Y}|/\delta)}{np}}.
       \label{th1.2}
   \end{align}
  Let $k=\frac{np}{1+c_0}$. Then,  $k\leq \frac{n \cdot A}{1+c_0}\leq T$, as $T$ satisfies \ref{conf_seq}. By (\ref{th1.2}) we have that $np \geq 4c_2^2 \log(n/\delta)$ and thus $k\geq \log(n/\delta)$. As a result  $np = (1+c_0)k \geq k+c_0 \max\{k,\log( n/\delta)\}$ , and by the first property, the ball $B=B(x,r_{p}(x))$ has $\#_n(B) \geq k$. Let $B_k(x)$ represent the ball with exactly the $k$ nearest neighbors of $x$

By the second property, we have that:
   \begin{align*}
       &P_n(\mathcal{S}_i|B_{k}(x))-  P_n(\mathcal{S}_{y_0}|B_{k}(x)) \geq   P(\mathcal{S}_i|B_{k}(x))-  P(\mathcal{S}_{y_0}|B_{k}(x)) -\Delta(n,k,\delta)  \geq   \gamma -\Delta(n,k,\delta) \geq  \\
        &2c_2 \sqrt{\frac{\log(|\mathcal{Y}|n/\delta)}{np}}- c_1 \sqrt{\frac{\log(|\mathcal{Y}|n/\delta)}{k}} \geq  2c_1 \sqrt{\frac{\log(|\mathcal{Y}|n/\delta)}{k}}-  c_1 \sqrt{\frac{\log(|\mathcal{Y}|n/\delta)}{k}} \geq\\
       & c_1 \sqrt{\frac{\log(|\mathcal{Y}|n/\delta)}{k}} =
        \Delta(n, k,\delta) \hspace{10mm} \forall i \in  \argmax_{y\in \mathcal{Y}}P(\mathcal{S}_y|x).
        \end{align*}

Therefore, we  have that the difference between the empirical frequencies of the elements in $\argmax_{y\in \mathcal{Y}}P(\mathcal{S}_y|x)$ and $y_0$  in the bags of $B_{k}(x)$ is at least as big as the threshold $\Delta$. If the algorithm has not eliminated all the labels  in $\argmax_{y\in \mathcal{Y}} P(\mathcal{S}_y|x)$  before the $k$th iteration, it will eliminate $y_0$ from the set of possible labels $\hat{s}$, so that the result is obtained. 

At the same time, for any ball $B'=B_{k'}(x)$ with $k'<k$ we have
\begin{align*}
    & P_n(\mathcal{S}_j|B')-P_n(\mathcal{S}_i|B') \leq  P(\mathcal{S}_j|B')-P(\mathcal{S}_i|B')+\Delta(n,\#_n(B'),\delta) <\Delta(n,\#_n(B'),\delta)\\ & \forall i \in \argmax_{y\in \mathcal{Y}}  P(\mathcal{S}_y|x), \forall j \in \mathcal{Y} \setminus \argmax_{y\in \mathcal{Y}}  P(\mathcal{S}_y|x).
    \end{align*}

  Therefore, although labels from $\argmax_{y\in \mathcal{Y}}  P(\mathcal{S}_y|x)$ can eliminate other labels of the same set before the $k$th iteration, we can always ensure that one of them will remain in the set of possible labels, since elements  in $\mathcal{Y} \setminus \argmax_{y\in \mathcal{Y}}  P(\mathcal{S}_y|x)$ cannot eliminate then. 
  
  The second part of the theorem is obtained using Lemma \ref{lemma1}.
\end{proof}

\subsection{Proof of Theorem \ref{uniform}}
\begin{proof}[Proof of Theorem \ref{uniform}]    
  This proof is strictly  analogous  to the proof of Theorem \ref{th1}. The only  difference is  that the following two properties

\begin{itemize}
    \item For any integer $k$ we have $\#_n(B) \geq k$ whenever $nP(B)\geq k+c_0 \max\{k,\log(n/\delta)\}$.
    \item $ |P_n(\mathcal{S}_y|B)-P(\mathcal{S}_y|B))| \leq \frac{1}{2} \Delta(n,\#_n(B),\delta) ,\hspace{3mm} \forall y\in \mathcal{Y}$, \hspace{3mm} and therefore  \\
    $|(P_n(\mathcal{S}_i|B)-P_n(\mathcal{S}_j|B))-(P(\mathcal{S}_i|B)-P(\mathcal{S}_j|B))| \leq \Delta(n,\#_n(B),\delta) , \forall i,j\in \mathcal{Y}$.

\end{itemize}

hold with probability $1-\delta^2$ for all the balls in $\mathcal{X}$
\end{proof}

\subsection{Proof of Theorem \ref{rates}}
Before providing the proof of the theorem we state and prove  a technical lemma.
\begin{lemma} \label{lemma_rates} Let $C$ be the constant from Theorem \ref{th1}.  Let $h$ be the classifier from  PL A-$k$NN algorithm \ref{alg1}  with  maximum  number of iterations  $T$   and  confidence parameter $\delta$ satisfying \ref{conf_seq}.  If the underlying  distribution satisfies the  label-aligned  condition \eqref{partial_condition} and $a>0$ satisfies 
\begin{align*} 
n \geq \frac{C}{a }\max\{\log(1/a),\log(|\mathcal{Y}|/\delta)\},
    \
\end{align*}
 we have 
   \begin{align*}
      & R(h)-R^* \leq \delta + P(\adv(x) \leq a)
         \end{align*}
with probability at least $1-\delta$.
\end{lemma}
\begin{proof} This proof is a simple application of Markov's inequality.

From   Theorem  \ref{th1} we have that  for each $x\in \mathcal{X}$ such that $\adv(x)> a$,  $Pr_n(h(x)\neq h^*(x)) \leq \delta^2 $,
where $Pr_n$ denotes probability over the choice of training points.  Thus,  for $\mathcal{X} \sim P$
$$\E_n \E_{\mathcal{X}}\mathbb{I}\{h(x) \neq h^*(x)| \adv(x) > a\} \leq \delta^2,$$
and by Markov’s inequality:
$$Pr_n [P(h(x)\neq h^*(x)| \adv(x) > a)\geq \delta] \leq \delta. $$Thus, with probability $1-\delta$ over the training examples:
$$P(h(x) \neq h^*(x)| \adv(x) >a)\leq \delta,$$

and we can conclude then that, with probability $1-\delta$ over the training examples,
\begin{align*}
   & R(h)-R^* \leq  P( \adv(x) \leq a)   +P(h(x) \neq h^*(x)| \adv(x) > a)  
    \leq \delta +  P( \adv(x) \leq a).
\end{align*}
\end{proof}

\begin{proof}[Proof of Theorem \ref{rates}]
    In the following, we show
\begin{align*} 
& n\geq \frac{C}{a_n }\max\{\log(1/a_n),\log(|\mathcal{Y}|/\delta)\},
\end{align*}
and such an inequality leads to the result by using Lemma \ref{lemma_rates}.

We have  that $a_n \geq C \frac{2 \log(n)}{n}$ from its definition. If we divide both sides by $\log(1/a_n)$, and since $a_n\geq C \frac{2 \log(n)}{n}$ implies that $n>\frac{1}{a_n}$, we have that  $\frac{a_n}{\log(1/a_n)} \geq \frac{C}{n}$. Therefore,  $n\geq C\frac{\log(1/a_n)}{a_n} $.

We also have that $a_n\geq C\frac{\log(|\mathcal{Y}|/\delta)}{n}$  from its definition, thus we have  that $n\geq C\frac{\log(|\mathcal{Y}|/\delta)}{a_n} $.

Computing the maximum of the previous two inequalities we obtain 
\begin{align*} 
n\geq \frac{C}{a_n }\max\{\log(1/a_n),\log(|\mathcal{Y}|/\delta)\},
\end{align*}
and the proof is concluded.
\end{proof}

\subsection{Proof of Theorem \ref{explicit_rates}}

\begin{proof}
Let $x$  be an instance from $ \mathcal{X}$. Choose   $t_0>0$ such that:
\begin{align*}
    &  P(\mathcal{S}_i|x)-P(\mathcal{S}_j|x) > t_0 & \forall  i \in \argmax_{y\in \mathcal{Y}} P(\mathcal{S}_y|x),  \forall j \notin \argmax_{y\in \mathcal{Y}} P(\mathcal{S}_y|x).
\end{align*}
Since   $P$ is $(\alpha,L)$-smooth,  choosing  $p_0= (\frac{t_0}{4L})^{1/\alpha}$ we have  :
\begin{align*}
   & P(\mathcal{S}_y|(B(x,r_p(x))) \geq  P(\mathcal{S}_y|x) -  L p^\alpha  &  \forall p\leq p_0, \forall y \in \mathcal{Y} \\
   & P(\mathcal{S}_y|(B(x,r_p(x)))  \leq  P(\mathcal{S}_y|x) +  L p^\alpha &  \forall p\leq p_0, \forall y \in \mathcal{Y}
\end{align*}
and therefore :
\begin{align*}
  & P(\mathcal{S}_i|(B(x,r_p(x)))-P(\mathcal{S}_j|(B(x,r_p(x))) \geq P(\mathcal{S}_i|x)-P(\mathcal{S}_j|x) -2 L p^\alpha \geq t_0 /2\\
  & \forall p\leq p_0,  i \in \argmax_{y\in \mathcal{Y}} P(\mathcal{S}_y|x), j \notin \argmax_{y\in \mathcal{Y}} P(\mathcal{S}_y|x).
\end{align*}
It is straightforward then that $\adv(x) >\min\{A, p_0\}\cdot t_0^2/4= \min\{At_0^2/4,\frac{t_0^{2+1/\alpha}}{4(4L)^{1/\alpha}}\}$.

 Since $P$ satisfies the $\beta$-margin condition, 
we have that 
\begin{align*}
     &P(\{ x\in \mathcal{X}: \forall i \in \argmax_{y\in Y} P(\mathcal{S}_y|x) \hspace{2mm} \exists  j \notin \argmax_{y\in Y}  P(\mathcal{S}_y|x) \hspace{2mm} st: P(\mathcal{S}_i|x)- P(\mathcal{S}_j|x) \leq t\}) \leq C_2\cdot t^\beta
     \end{align*}
     and consequently 
\begin{align*}
     &P(\{ x\in \mathcal{X}: \forall i \in \argmax_{y\in Y} P(\mathcal{S}_y|x) \hspace{2mm} \forall  j \notin \argmax_{y\in Y}  P(\mathcal{S}_y|x) \hspace{2mm}  st:P(\mathcal{S}_i|x)- P(\mathcal{S}_j|x) > t\}) \geq 1- C_2 t^\beta.
     \end{align*}
Therefore, using the implication about the advantage derived at the beginning of the proof we have that $P(\adv(x)>\min\{\frac{At_0^2}{4},\frac{t_0^{2+1/\alpha}}{4(4L)^{1/\alpha}}\}) \geq  1- C_2t^\beta  $, so $P(\adv(x)\leq \min\{\frac{At_0^2}{4},\frac{t_0^{2+1/\alpha}}{4(4L)^{1/\alpha}}\}) \leq   C_2 t^\beta  $ for any $t>0$, which  we can rewrite as  $P(\adv(x)\leq a) \leq   C_2'\cdot (a^{\frac{\alpha}{2\alpha+1}})^\beta  $ for any $a \in [0,1]$, where $C_2'=C_2\ ((\frac{4}{A})^{\alpha}4L)^{\frac{1}{2\alpha+1}})^\beta$.

Therefore, we have from Theorem \ref{rates} that  with probability $1-\delta$:
\begin{align*}
      & R(h)-R^* \leq \delta +P(\adv(x)\leq a_n) \leq \delta +  C_3 \left( \frac{1}{n}\max \Bigl\{\log(n),\log(|\mathcal{Y}|/\delta)\Bigl\}\right) ^{ \beta(\frac{\alpha}{2\alpha+1})}
\end{align*}
 where $C_3=(2C)^{\beta(\frac{\alpha}{2\alpha+1})}\cdot C_2'$, which concludes the proof.
\end{proof}

\subsection{Proof of Theorem \ref{consistency}}
\begin{proof}[Proof of Theorem \ref{consistency}]
  Given the sequence of confidence parameters $\{ \delta_n\}_{n=1}^\infty$, we define a sequence of advantage values as in Theorem \ref{rates}:
\begin{align*}
     a_n=\frac{C}{n} \max \bigl\{2\log(n),\log(|\mathcal{Y}|/\delta_n) \bigl\}.
\end{align*}
The conditions established over the sequence $\{ \delta_n\}_{n=1}^\infty$   imply that $a_n\rightarrow0$.

If $\epsilon>0$, by the conditions  established for $\{ \delta_n\}_{n=1}^\infty$ we can choose an $N$ such that $\sum_{n=N}^\infty \delta_n < \epsilon$. Let  $\{ (x_n, s_{n})\}_{n=1}^\infty $ denote a random sequence of training examples. We have by Theorem \ref{rates} that:
\begin{align*}
    &P( \exists n \geq N:  R(h_n)-R^* \geq \delta + P(\adv(x) \leq a_n) ) \\
    &\leq \sum_{n=N}^\infty P(  R(h_n)-R^* \geq \delta + P(\adv(x) \leq a_n) ) \\
    &\leq \sum_{n=N}^\infty \delta_n < \epsilon.
\end{align*}
Thus, with probability at least $1-\epsilon $ over the training sequence  $\{ (x_n, s_{n})\}_{n=1}^\infty $ we have that $ \forall n \geq N$
 $$R(h_n)-R^* \leq \delta_n + P(\adv(x) \leq a_n) $$
Therefore, since  $a_n\rightarrow0$, and   $\lim_{a_n\rightarrow 0}P(\adv(x) \leq a_n)=0$  (see  Lemma \ref{lemma1}) we have  that  $R(h_n) \rightarrow R^*$ almost surely.

 We  now proof the second part of the  theorem, which states that no other   algorithm can achieve  Bayes consistency under  more general scenarios  of  PLL than PL A-$k$NN. We   proof that  statement showing that if an algorithm is  Bayes consistent for a  case where PL A-$k$NN fails, then that  algorithm fails in a case with label-aligned bags, where PL A-$k$NN is Bayes consistent.

Let $P_1$ be a distribution in $\mathcal{X}\times\mathcal{Y}\times\mathcal{S}$ for which PL A-$k$NN is not Bayes consistent. Then $P_1$ that does  not  have label-aligned bags  on a subset of instances 
\[
B = \{x \in \mathcal{X} : \argmax_{y\in\mathcal{Y}} P_1(\mathcal{S}_y|x) \neq \argmax_{y\in\mathcal{Y}} P_1(y|x)\},
\]
with $P_1(B) > 0$, since if $P_1(B)=0$, Theorem \eqref{limit_any_dist} already guarantees that PLA-kNN is Bayes consistent.
 We construct a second distribution $P_2$  with label-aligned bags  that has the  same distribution  over bags than $P_1$ but their associated Bayes  classifiers have different values for each  instance in $B$.

We keep the same instance marginal for both distributions:
\[
P_{2}(x) = P_{1}(x), \quad \forall x \in \mathcal{X}.
\]
For each $x \notin B$, $P_1$ and $P_2$ have the same distributions over labels $P_2(y|x) = P_1(y|x)$ and  bag generation process $P_2(s|y,x) = P_1(s|y,x)$. 
For each $x \in B$, let
\[
y_1 = \argmax_{y \in \mathcal{Y}} P_1(y|x), \qquad 
y_2 = \argmax_{y \in \mathcal{Y}} P_1(\mathcal{S}_y|x),
\]
where it is clear that $y_1 \neq y_2$. 
We   then \emph{flip} the label distribution and bag-generation process in $P_2$  with respect to $P_1$ for  $y_1$ and $y_2$,while keeping all other components equal in both distributions. Specifically,
\begin{align*}   
& P_2(y_1|x) = P_1(y_2|x), \quad P_2(y_2|x) = P_1(y_1|x)\\
&P_2(s|y_1,x) = P_1(s|y_2,x), \quad P_2(s|y_2,x) = P_1(s|y_1,x)\\
& P_2(y|x) = P_1(y|x), \hspace{7mm} 
P_2(s|y,x) = P_1(s|y,x) & \forall s \in \mathcal{S},  \forall y \in \mathcal{Y} \setminus \{y_1,y_2\}.
\end{align*}

For any $x \in \mathcal{X}$ and $s \in \mathcal{S}$,
\[
P_2(s|x) = \sum_{y \in \mathcal{Y}} P_2(y|x) P_2(s|y,x)
         = \sum_{y \in \mathcal{Y}} P_1(y|x) P_1(s|y,x)
         = P_1(s|x),
\]
 as the sum has the same $|\mathcal{Y}|$ components, but in a different order  when $x \in B$; so  it follows that $P_2(x,s) = P_1(x,s)$.  Therefore, $P_1$ and $P_2$ induce \emph{exactly the same distribution} over the observable pairs $(x,s)$; any algorithm that learns only from $(x,s)$-samples cannot distinguish between them. However, by construction, the Bayes-optimal classifiers differs on $B$:
\[
h_{P_1}^*(x) = \argmax_y P_1(y|x) = y_1, \qquad 
h_{P_2}^*(x) = \argmax_y P_2(y|x) = y_2.
\]Moreover,  the distribution $P_2$ has label-aligned bags, since  for any $x \in B$ where the condition did not hold for $P_1$
\[
y_2 = \argmax_{y \in \mathcal{Y}} P_2(y|x), \hspace{2mm} \text{and} \hspace{2mm} 
y_2 = \argmax_{y \in \mathcal{Y}} P_1(\mathcal{S}_y|x)= \argmax_{y \in \mathcal{Y}} P_2(\mathcal{S}_y|x).
\]

 The  Bayes rules corresponding to $P_1$ and $P_2$ are distinct on a set of nonzero probability, as  $P_{1}(B) > 0$ and $h_{P_1}^*(x)\neq h_{P_2}^*(x)$ for any instance in $B$. Since  $P_1(x,s) = P_2(x,s)$, any consistent algorithm would converge to the same rule under both $P_1$ and $P_2$. 
Yet, as Bayes rules for $P_1$ and $P_2$ differ in a set of non-zero measure,  no algorithm can achieve the Bayes risk for both cases.  Therefore, if an algorithm is Bayes consistent for $P_1$ it  can't be consistent for $P_2$, a scenario with label-aligned bags. Hence no algorithm can be Bayes consistent under  strictly more general  scenarios than those of PL A-$k$NN.\end{proof}

\subsection{Proof of Theorem \ref{limit_any_dist}}
Before providing the proof of the theorem we state and prove some technical lemmas.
\begin{lemma} \label{relax_query_rates}
(Query-dependent convergence) There is an absolute  constant $C>0$ for which the following holds. Let $h$ be the classifier from  PL A-$k$NN algorithm \ref{alg1}  with  maximum number of iterations $T$ and  confidence parameter $\delta$  satisfying \eqref{conf_seq}.  If the underlying distribution   satisfies the relaxed label-aligned condition \eqref{relax} for $G\subseteq \mathcal{X}$ and  $\theta$,  the classifier $h$ satisfies the following : for each $x \in G$ such that 
\begin{align*} 
    n \geq \frac{C}{\adv(x) }\max\{ \log(1/\adv(x)),\log(|\mathcal{Y}|/\delta)\},
    \
\end{align*}
we have  that   $h(x)\in\mathcal{Y}^\theta(x)$  with probability at least $1- \delta^2$.

In addition, for each  $x \in  \mathcal{X}\setminus G$ such that 
\begin{align*} 
    n \geq \frac{C}{\adv(x) }\max\{ \log(1/\adv(x)),\log(|\mathcal{Y}|/\delta)\},
    \
\end{align*}
we have that  $h(x)= h^*(x)$  with probability at least $1- \delta^2$.
\end{lemma}

\begin{proof}
    The proof is strictly analogous to the proof of Theorem \ref{th1}, since for the elements  $x\in G$  we have that $\argmax_{y \in \mathcal{Y}}P(\mathcal{S}_y|x) \subseteq  \mathcal{Y}^\theta(x)$.
\end{proof}

\begin{lemma} \label{relax_rates}
  (Rates of convergence) Let $C$ be the constant from Theorem \ref{th1}, and  $h$ be the classifier from  PL A-$k$NN algorithm \ref{alg1} with  maximum number of iterations $T$ and  confidence parameter $\delta$ satisfying \eqref{conf_seq}.  If the underlying distribution  satisfies the   relaxed  label-aligned condition  \eqref{relax} for $G\subseteq \mathcal{X}$ and  $\theta$ ,  we have 
   \begin{align*}
      & R(h)-R^* \leq \delta + P(\adv(x) \leq a_n) + \theta \cdot P(G)\\
      & \text{where } \hspace{1mm} a_n=\frac{C}{n} \max \bigl\{2\log(n),\log(|\mathcal{Y}|/\delta) \bigl\}.
\end{align*}
with probability at least $1-\delta$.
\end{lemma}
\begin{proof}

    Let $P_G$ denote the probability distribution restricted to the set $G$ and  $R_G(h)$ denote de risk of a classification rule over $P_G$. Then , we have that:
    
   $$ R(h)=  R_G(h)P(G)+ R_{\mathcal{X}\setminus G}(h)(1-P(G)).$$

From the proof of  Theorem \ref{rates},  and Lemma \ref{relax_query_rates}  we have that for  each  $x\in \mathcal{X}\setminus G$ such that  $\adv(x)> a_n$ we have  $Pr_n(h(x)\neq h^*(x)) \leq \delta^2 $,
where $Pr_n$ denotes probability over the choice of training points. Thus, for  $\mathcal{X}\setminus G \sim P_{\mathcal{X}\setminus G}$

$$\E_n \E_{\mathcal{X}\setminus G}\mathbb{I}\{h(x)\neq h^*(x)| \adv(x) > a_n\} \leq \delta^2$$

and by Markov’s inequality:
$$Pr_n [P_{\mathcal{X}\setminus G}(h(x)\neq h^*(x)| \adv(x) > a_n)\geq \delta] \leq \delta. $$

Thus, with probability $1-\delta$ over the training examples:
$$P_{\mathcal{X}\setminus G}(h(x)\neq h^*(x))| \adv(x) >a_n\leq \delta.$$
Therefore
\begin{align*}
      &R_{\mathcal{X}\setminus G}(h)-R_{\mathcal{X}\setminus G}^* \leq  P_{\mathcal{X}\setminus G}( h(x) \neq h^*(x) | \adv(x) > a_n) +P_{\mathcal{X}\setminus G}(\adv(x) \leq a_n)  \\
& \leq \delta +P_{\mathcal{X}\setminus G}(\adv(x) \leq a_n).
\end{align*}

 From the proof of  Theorem \ref{rates} and Lemma \ref{relax_query_rates}  we have that for each $x\in G$ such that $\adv(x)> a_n$ then  $Pr_n(h(x)\notin \mathcal{Y}^\theta(x)) \leq \delta^2 $,
where $Pr_n$ denotes probability over the choice of training points. Thus, for $G \sim P_G$
$$\E_n \E_G\mathbb{I}\{h(x) \notin\mathcal{Y}^\theta(x)| \adv(x)> a_n\} \leq \delta^2$$
and by Markov’s inequality:
$$Pr_n [P_G(h(x)\notin \mathcal{Y}^\theta(x)| \adv(x) > a_n\geq \delta] \leq \delta. $$
Thus with probability $1-\delta$ over the training examples:
$$P_G(h(x)\notin \mathcal{Y}^\theta(x))| \adv(x) >a_n)\leq \delta.$$
Therefore,
\begin{align*}
   & R_G(h)-R_G^* \leq  P_G( \adv(x) \leq a_n) + \theta P_G(h(x)\in \mathcal{Y}^\theta(x)| \adv(x) > a_n) \\ 
   &  +P_G(h(x)\notin \mathcal{Y}^\theta(x)| \adv(x) > a_n)  
    \leq \delta + \theta +  P_G( \adv(x) \leq a_n).
\end{align*}

We then conclude the proof since 
\begin{align*}
   & R(h) -R^*=  (R_G(h)-R_G^*)P(G)+ (R_{\mathcal{X}\setminus G}(h)-R^*_{\mathcal{X}\setminus G})(1-P(G))  \\
    &\leq \delta + \theta P(G) +P_G( \adv(x) \leq a_n) P(G)+P_{\mathcal{X}\setminus G}( \adv(x) \leq a_n)(1-P(G))\\
   &  \leq \delta + P(\adv(x) \leq a_n) +\theta P(G).
\end{align*}
\end{proof}
\begin{proof}[Proof of Theorem \ref{limit_any_dist}]
The proof is  analogous to the proof of Theorem \ref{consistency}, but using the rates from Lemma \ref{relax_rates} instead of the rates from Theorem \ref{rates}.
\end{proof}

\newpage
\section{Additional experimental and implementation details}
\subsection{Experimental results}

\begin{figure*}[!h]
    \centering
    \begin{subfigure}[b]{0.32\textwidth}
        \centering
        \includegraphics[width=\linewidth]{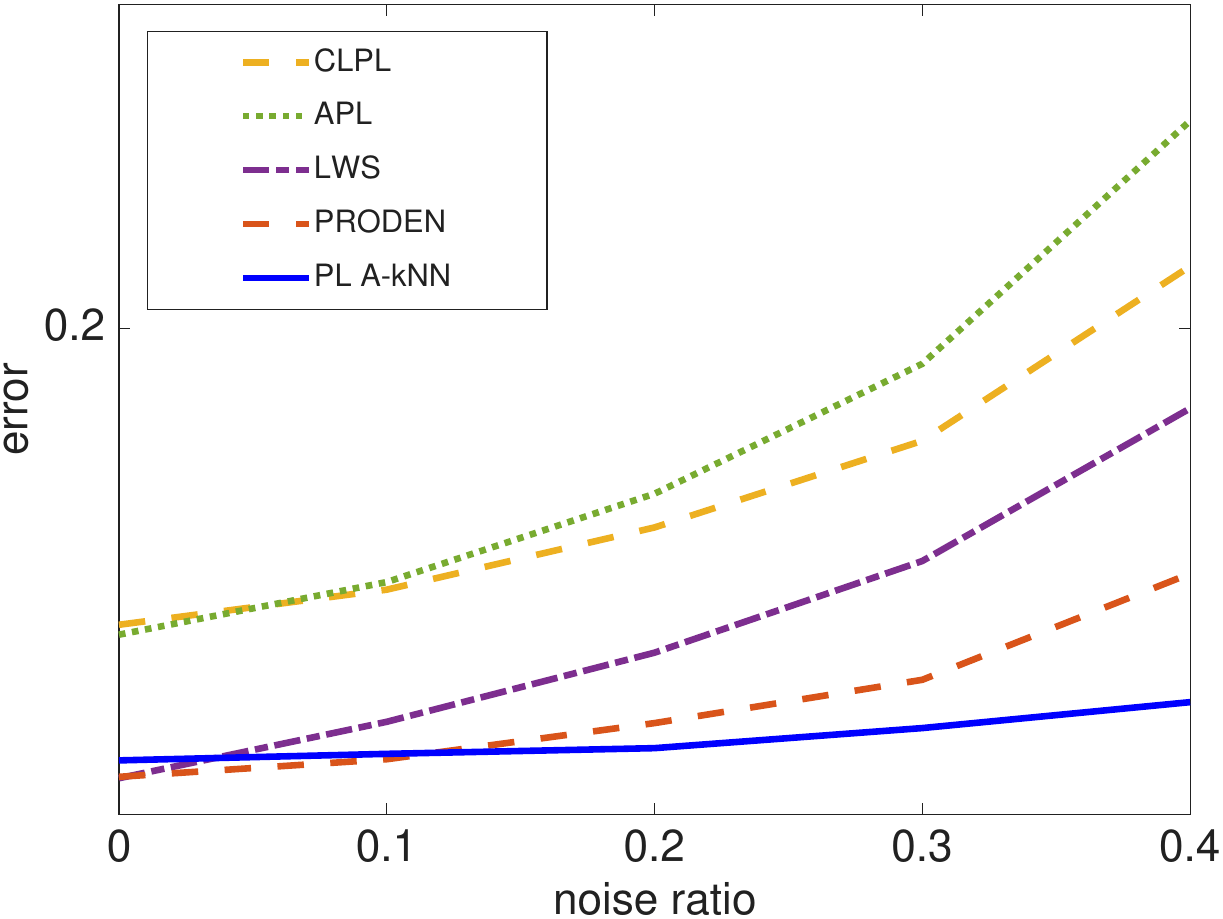}
        \caption{MNIST}
    \end{subfigure}
    \hfill
    \begin{subfigure}[b]{0.32\textwidth}
        \centering
        \includegraphics[width=\linewidth]{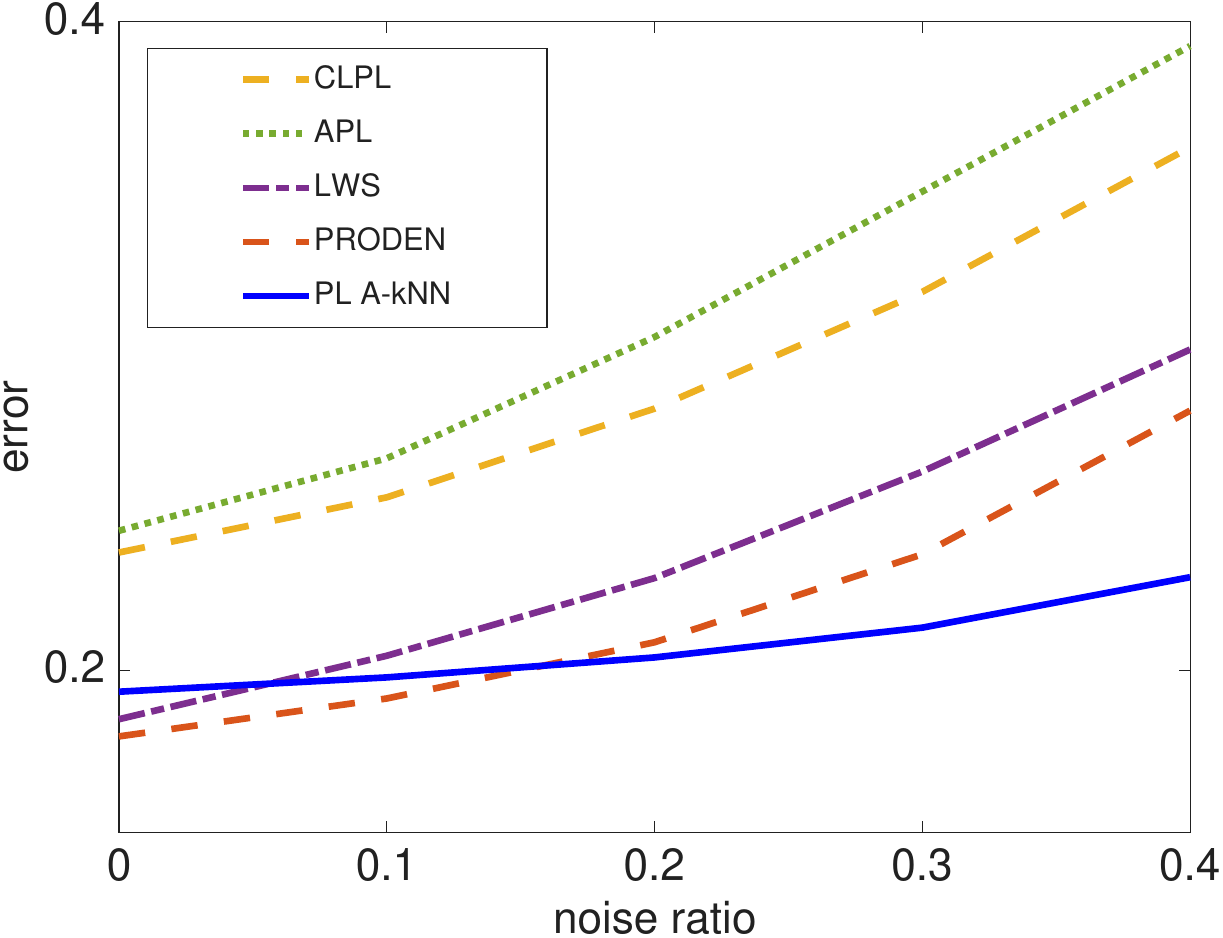}
        \caption{CIFAR-10}
    \end{subfigure}
    \hfill
    \begin{subfigure}[b]{0.32\textwidth}
        \centering
        \includegraphics[width=\linewidth]{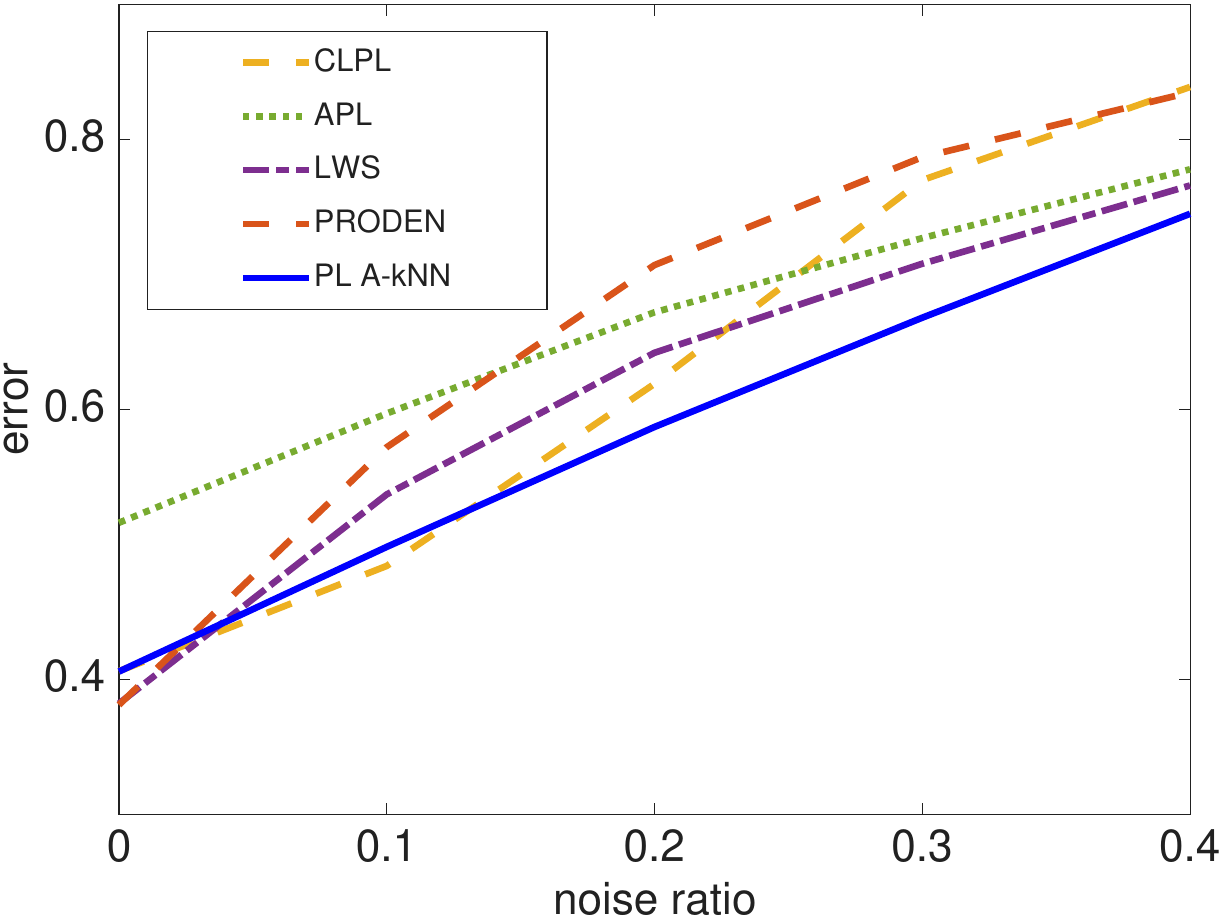}
        \caption{MirFlickr}
    \end{subfigure}
    \caption{Comparison of the error rates of PL A-$k$NN and state-of-the-art methods  for MNIST, CIFAR-10, and MirFlickr under an increasing noise rate. The results show that PL A-$k$NN  outperforms existing approaches across a wide range of noise levels. }
    
\end{figure*}

\begin{figure*}[!h]
    \centering
    \begin{subfigure}[b]{0.32\textwidth}
        \centering
        \includegraphics[width=\linewidth]{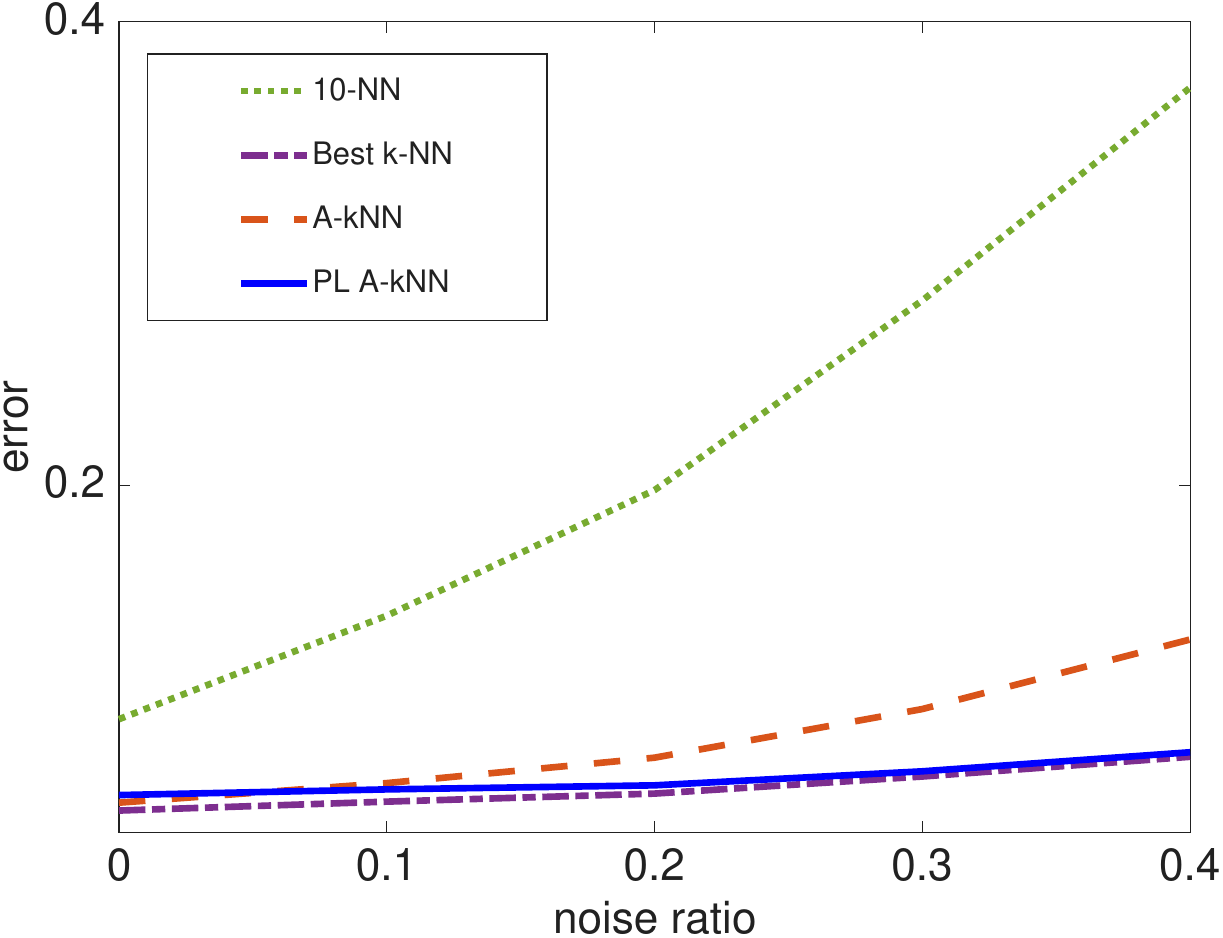}
        \caption{MNIST}
    \end{subfigure}
    \hfill
    \begin{subfigure}[b]{0.32\textwidth}
        \centering
        \includegraphics[width=\linewidth]{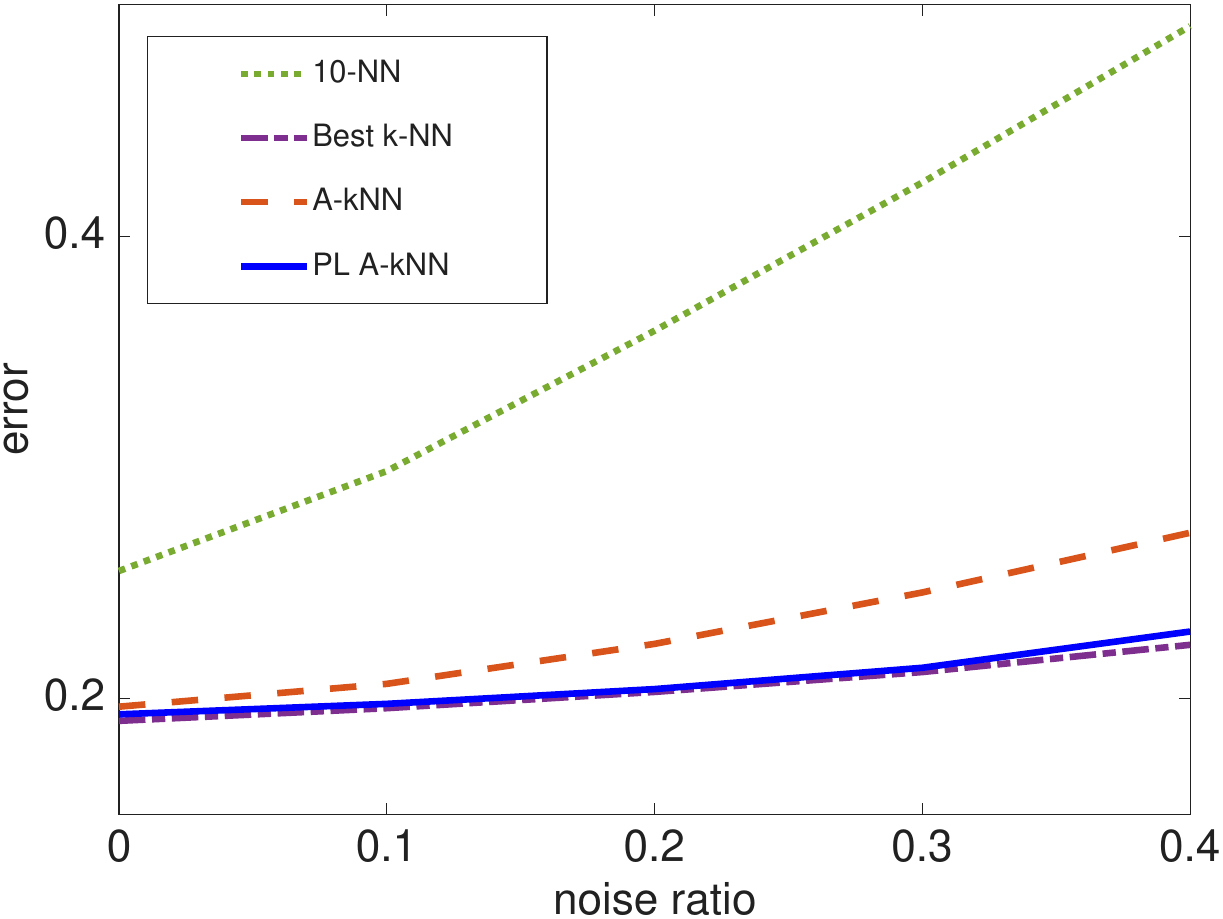}
        \caption{CIFAR-10}
    \end{subfigure}
    \hfill
    \begin{subfigure}[b]{0.32\textwidth}
        \centering
        \includegraphics[width=\linewidth]{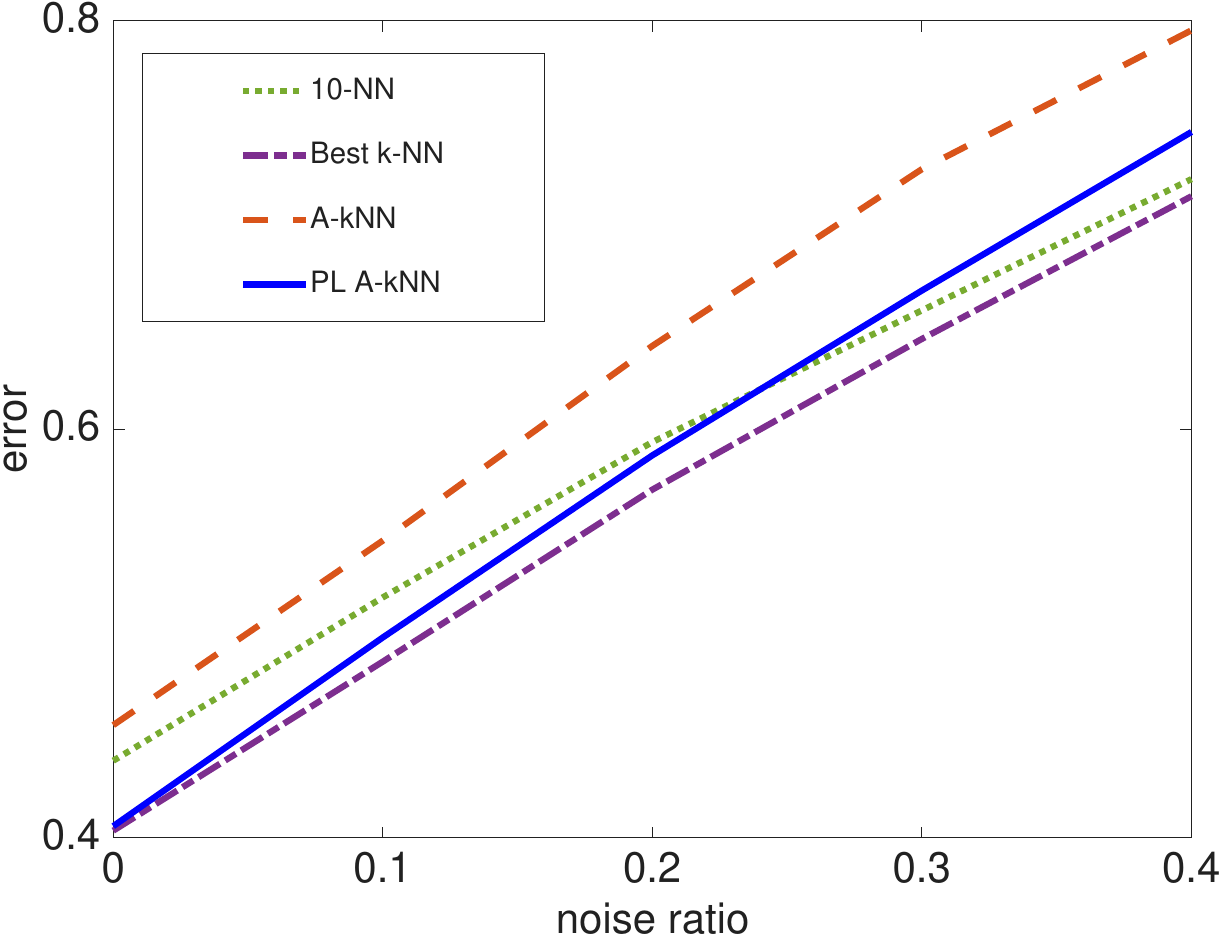}
        \caption{MirFlickr}
    \end{subfigure}
    \caption{Comparison of the error rates of PL A-$k$NN and $k$NN benchmarks for MNIST, CIFAR-10, and MirFlickr under an increasing noise rate. The results show that PL A-$k$NN outperforms $10$-NN and A-$k$NN across a wide range of noise levels, while having comparable performance to the best $k$-NN.}

\end{figure*}

\subsection{\texorpdfstring{Feature Preprocessing for PL A-$k$NN}{Feature Preprocessing for PL-kNN}}
We apply two preprocessing pipelines to obtain compact and
stable representations suitable for nearest-neighbor retrieval: one
designed for the vision benchmarks (CIFAR-10, MNIST, and Fashion-MNIST),
and one tailored for the real-world partially labeled datasets (MirFlickr and
MRSCv2). Both pipelines share a common structure---centering,
$\ell_2$-normalization, Gaussian-weighted KNN smoothing, and local
density point transform---and differ only in that the real-world pipeline
utilizes a signed cube-root transform in place of centering.
The preprocessing implementations are provided in  \href{https://github.com/MachineLearningBCAM/Adaptive-nearest-neighbours-for-partial-labels}{Adaptive Nearest Neighbours Repository}.

\paragraph{Vision datasets (CIFAR-10, MNIST, Fashion-MNIST).}
For the standard vision benchmarks we apply the following pipeline.
\begin{enumerate}
\item \textbf{Centering.}
Features are mean-centered using the training-set mean, which is then
applied without modification to the test set.

\item \textbf{$\ell_2$-normalization.}
The centered features are $\ell_2$-normalized, so that
similarity computations depend only on the direction of each vector.

\item \textbf{Gaussian-weighted KNN smoothing.}
Each feature vector is replaced by a convex combination of itself and
the Gaussian-weighted mean of its $10$ nearest neighbors,
\begin{equation}
    \tilde{\mathbf{x}}_i
    = (1-\alpha)\,\mathbf{x}_i
    + \alpha \sum_{j \in \mathcal{N}(i)} w_{ij}\,\mathbf{x}_j,
    \qquad \alpha = 0.25,
\end{equation}
where the weights $w_{ij} \propto \exp\!\bigl(-d_{ij}^{2}/(2\sigma_i^{2})\bigr)$
use a \emph{local} bandwidth $\sigma_i$ set to the median pairwise distance
from point $i$ to its $10$ neighbors.
The smoothed vectors are re-normalized to unit $\ell_2$-norm.
Test features are smoothed using neighbors found in the
\emph{already-smoothed} training set, and are likewise re-normalized.

\item \textbf{Local density point transform.}
To mitigate hubness in high-dimensional spaces, each feature vector is
rescaled by its local neighbourhood density. Specifically, each point
$\mathbf{x}_i$ is divided by $r_k(i)$, the mean distance to its $k=50$
nearest neighbors,
\begin{equation}
    \tilde{\mathbf{x}}_i = \frac{\mathbf{x}_i}{r_k(i)},
    \qquad
    r_k(i) = \frac{1}{k}\sum_{j \in \mathcal{N}(i)} d_{ij}.
\end{equation}
For test points, $r_k(q)$ is estimated using neighbors found in the
smoothed training set. No $\ell_2$-renormalization is applied after
this step. The transformed vectors
are used directly for Euclidean nearest-neighbor retrieval.
\end{enumerate}

\paragraph{Real-world datasets (MirFlickr, MRSCv2).}
For the two real-world partially labeled benchmarks, which provide
pre-extracted, high-dimensional descriptors, we use a slightly adapted
pipeline. Steps 2--4 are identical to the vision pipeline above, the only difference is
that centering is replaced by a nonlinear transform applied at the very
beginning.
\begin{enumerate}
\item \textbf{Signed cube-root transform.}
A signed cube-root transformation,
\begin{equation}
    x \;\leftarrow\; \operatorname{sign}(x)\,|x|^{1/3},
\end{equation}
is applied element-wise to compress large activations and equalize
feature variance, enhancing the similarity structure of the
representation~\citep{perronnin2010improving}.

\item \textbf{$\ell_2$-normalization, KNN smoothing ($\alpha=0.1$,
$k=10$), and local density point transform ($k=100$)} follow
identically as in steps 2--4 of the vision pipeline.
\end{enumerate}

In both pipelines, all statistics used for centering and bandwidth
estimation are computed exclusively on the training set and applied
without modification to the test set, preventing any leakage of test
information into the preprocessing stage.

\subsection{\texorpdfstring{PL A-$k$NN: Handling multiple partial labels after $T$ iterations}{ PL A-kNN: Handling multiple partial labels After T Iterations}}
In the cases where the set of possible labels $\hat{s}$ of an instance contains more than one label after $T$ iterations, we need a criterion to select a single label from $\hat{s}$. To this end, we adopt a heuristic that selects the label in $\hat{s}$ which was \emph{closest to eliminating} all other possible labels during the iterative process. Specifically, we select
\begin{align*}
   i = \argmin_{y \in \hat{s},\, k \leq T} 
   \sqrt{k} \left( 
   \Delta(n,k,\delta)
   - (P_n(\mathcal{S}_y | B_k(x)) - m_2(k))
   \right),
\end{align*}
where $B_k(x)$ denotes the ball containing the $k$ nearest neighbors of instance $x$, and $m_2(k)$ represents the frequency of the second most frequent label within the bags of the $k$ nearest neighbors. 

\pagebreak
The criterion  presented above measures how close the difference between the frequencies of each of  the labels in $\hat{s}$ and the second most frequent label is to surpassing the decision threshold across an increasing neighborhood size. Note that, if any of the  labels  in $\hat{s}$ actually surpassed the threshold, it would eliminate all other labels from $\hat{s}$. Since the threshold $\Delta$ decreases as $k$ increases (with order $\sqrt{k}$), we scale the difference by $\sqrt{k}$ to have a more fair  comparison across different neighborhood sizes. The detailed pseudocode for this selection criterion is provided in Algorithm~\ref{alg2}. The algorithm implementation can be found in \href{https://github.com/MachineLearningBCAM/Adaptive-nearest-neighbours-for-partial-labels}{Adaptive Nearest Neighbours Repository}.

\begin{algorithm}[!ht]
\caption{\textbf{PL A-$k$NN algorithm with disambiguation criterion}}\label{alg2}
 \textbf{Inputs:}  Instance  $x$   \\ 
 Training examples $\{ (x_l, s_{l})\}_{l=1}^n $ \\
   Maximum iterations $T$  \\
   Confidence parameter $\delta$ \\
\textbf{Output:}  Label  $h(x) $

\begin{algorithmic}[1]
\STATE Set $A=c_1 \sqrt{\log(n) +\log(|\mathcal{Y}|/\delta})$  
\STATE Initialize $\hat{s}=\mathcal{Y}$ and  $k=0$
\STATE Initialize  $\tau_1,\tau_2, \cdots, \tau_{|\mathcal{Y}|} =0$ and $M = \mathbf{0} \in \R^{T \times |\mathcal{Y}|}$ 
\WHILE{$|\hat{s}|>1$ and $k< T $} 
\STATE $k=k+1$
\STATE Find the $k$ nearest neighbor of $x$ and take $l_k$ as its index in $l=1,2,...n$
\STATE $\Delta=\frac{A}{\sqrt{k}}$
\STATE $\tau_y=\tau_y + \mathbb{I}\{ y \in s_{l_k} \}$ \hspace{3mm} $ \forall y\in \mathcal{Y}$

\STATE $\{m_1, m_2\} = \underset{y \in \hat{s}}{\text{max2}}\, \tau_y $\COMMENT{max2  returns the two largest values}
\FOR{$ y \in  \hat{s}$}
\STATE $M(k,y)=\sqrt{k}(\Delta-\frac{\tau_y-m_2}{k})$
\IF{$ \frac{m_1-\tau_y}{k}\geq \Delta $}
\STATE $\hat{s}=\hat{s} \setminus \{y\}$
\ENDIF

\ENDFOR

\ENDWHILE
\STATE Find  $\hat{y} = \underset{y \in \hat{s},\, k \leq T}{\argmin} M(k, y)$
\STATE $h(x)=\hat{y}$

\end{algorithmic}
\end{algorithm}

\end{document}